\documentclass{article}

\input{macros-arXiv.sty}

\usepackage{url}            
\usepackage{booktabs}       
\usepackage{amsfonts, bm, amssymb}       
\usepackage{nicefrac}       
\usepackage{microtype}      
\usepackage{tikz} 
\usepackage{float}

\usepackage{array}
\usepackage{diagbox}
\usepackage{mathtools}
\usepackage{multirow, makecell}
\usepackage{rotating}
\usepackage{graphicx}
\usepackage{subfigure}
\usepackage{mathdots}

\usepackage{float}
\usepackage{caption}
\usepackage{setspace}
\usepackage[letterpaper, left=1in, right=1in, top=1in, bottom=1in]{geometry}

\usepackage[bottom]{footmisc}

\usepackage[colorlinks,linkcolor={blue!75!black},urlcolor=red,citecolor=ForestGreen]{hyperref}

\usepackage{threeparttable}

\usepackage{arydshln}

\usepackage{enumerate}
\usepackage{enumitem} 

\usepackage{lipsum}
\usepackage{comment}

\usepackage{natbib}
\usepackage{amsmath}
\allowdisplaybreaks[4]

\title{\bf On The Convergence of Euler Discretization of Finite-Time Convergent Gradient Flows}

\author{
Siqi Zhang\thanks{Department of Applied Mathematics and Statistics, Johns Hopkins University, USA.}
\and 
Mouhacine Benosman\thanks{M. Benosman is with Amazon Robotics, Boston, MA (This work was completed prior to M. Benosman joining Amazon Robotics.)}
\and 
Orlando Romero\thanks{Department of Electrical and Systems Engineering, University of Pennsylvania, Philadelphia, PA, USA.}
}

\begin{document}

\maketitle

\begin{abstract}
In this study, we investigate the performance of two novel first-order optimization algorithms, namely the rescaled-gradient flow (RGF) and the signed-gradient flow (SGF). These algorithms are derived from the forward Euler discretization of finite-time convergent flows, comprised of non-Lipschitz dynamical systems, which locally converge to the minima of gradient-dominated functions. We first characterize the closeness between the continuous flows and the discretizations, then we proceed to present (linear) convergence guarantees of the discrete algorithms (in the general and the stochastic case). Furthermore, in cases where problem parameters remain unknown or exhibit non-uniformity, we further integrate the line-search strategy with RGF/SGF and provide convergence analysis in this setting. We then apply the proposed algorithms to academic examples and deep neural network training, our results show that our schemes demonstrate faster convergences against standard optimization alternatives.
\end{abstract}

\section{Introduction}
In this work, we consider the unconstrained minimization problem for a given cost function \mbox{$f:\mathbb{R}^n\to\mathbb{R}$}. When $f$ is sufficiently regular, the standard algorithm in continuous time (dynamical system) is given by
\begin{equation}
    \dot{x} = F_\mathrm{GF} (x)\triangleq -\nabla f(x)
    \label{eq:gradientflow}
\end{equation}
with $\dot{x}\triangleq\frac{\mathrm{d}}{\mathrm{d}t}x(t)$, known as the \emph{gradient flow} (GF). Generalizing GF, the $q$-\emph{rescaled} GF ($q$-RGF)~\cite{Wibisono2016} given by
\begin{equation}
    \dot{x} =F_{q-\mathrm{RGF}}(x)= -c\frac{\nabla f(x)}{\|\nabla f(x)\|_2^{\frac{q-2}{q-1}}},\ c>0,\ q\in (1,\infty]
    \label{eq:qRGF}
\end{equation}
has an asymptotic convergence rate $f(x(t))-f(x^\star) = \mathcal{O}\left(\frac{1}{t^{q-1}}\right)$ under mild regularity, for $\|x(0)-x^\star\|_2>0$ small enough, where $x^\star\in\mathbb{R}^n$ denotes a local minimizer of $f$. However,  recently \cite{ourICMLpaperanon} shows that $q$-RGF, as well as $q$-\emph{signed} GF ($q$-SGF):
\begin{equation}
    \dot{x}=F_{q-\mathrm{SGF}}(x) = - c\,\|\nabla f(x)\|_1^{\frac{1}{q-1}}\sign(\nabla f(x)),
    \label{eq:qSGF}
\end{equation}
where $\sign(\cdot)$ denotes the sign function (element-wise), are both finite-time convergent (in continuous-time), provided that $f$ is gradient dominated of order $p\in (1,q)$. 
Considering that many algorithms are inspired by continuous flows with convergence guarantee, e.g., \cite{Muehlebach2019,Fazlyab2017,Shi2018,Zhang2018,Franca2019b,Wibisono2016}, a natural question arises: \textit{what is the convergence rate of the corresponding discrete-time algorithms, which are induced by discretization of finite-time convergent continuous-time flows?}

\subsection{Contribution}
In this paper, we investigate the convergence behavior of an Euler discretization for the $q$-RGF~\eqref{eq:qRGF} and $q$-SGF~\eqref{eq:qSGF}. We provide convergence guarantees in terms of the closeness of solutions, using results from hybrid dynamical control theory. Furthermore, we provide the iteration complexity upper bounds of the proposed algorithms, to the best of our knowledge, this is the first result on the convergence guarantees of these algorithms in the (structural) nonconvex regime. We then test the performance of the proposed algorithms on both synthetic and real-world data in the context of deep learning, namely, on the well-known SVHN dataset, the experiment results show the outperformance of our proposed algorithms.

\subsection{Related Work}
Propelled by the seminal works~\cite{Elia2011,Su2014}, there has been a recent and significant research effort dedicated to analyzing optimization algorithms from the perspective of dynamical systems and control theory, especially in continuous time~\cite{Wibisono2016,Wilson2018,Lessard2016,Fazlyab2017b,Scieur2017,Franca2018,Fazlyab2018,Fazlyab2018b,Taylor2018,Franca2019,Orvieto2019,romero2019,Muehlebach2019}. A major focus within this initiative is in \emph{accceleration}, both in terms of trying to gain new insight into common optimization algorithms from this perspective, or even to exploit the interplay between continuous-time systems and their potential discretizations for novel algorithm design~\cite{Muehlebach2019,Fazlyab2017,Shi2018,Zhang2018,Franca2019b,wilson2019}. Many of these papers also focus on deriving convergence rates based on the discretization of flows designed in the continuous-time domain.

Connecting ordinary differential equations~(ODEs) with optimization algorithms is an important topic, which can be dated back to the 1970s, see \cite{botsaris1978,botsaris1978b,zghier1981,snyman1982,snyman1983,Brockett88,Brown1989}. 
In \cite{helmke1994}, the authors studied relationships between linear programming, ODEs, and general matrix theory.  Further,~\cite{Schropp1995} and \cite{schropp2000} explored several aspects linking nonlinear dynamical systems to gradient-based optimization, including nonlinear constraints.  

Lyapunov stability theory is often employed for the analysis, there already exists a rich body of works within the nonlinear systems and control theory community for this purpose. Typically in previous works, one seeks asymptotically Lyapunov stable gradient-based systems with an equilibrium (stationary point) at an isolated extremum of the given cost function, thus certifying local convergence. Naturally, the global asymptotic stability leads to global convergence, though such analysis will typically require the cost function to be strongly convex everywhere. 

For physical systems, a Lyapunov function can often be constructed from first principles via some physically meaningful measure of energy (\emph{e.g.}, total energy = potential energy + kinetic energy). In optimization, the situation is somewhat similar in the sense that a suitable Lyapunov function may often be constructed by taking simple surrogates of the objective function as candidates. For instance, $V(x) \triangleq f(x) - f(x^\star)$ can be a good initial candidate. Further, if $f$ is continuously differentiable and $x^\star$ is an isolated stationary point, then another alternative is $V(x) \triangleq \|\nabla f(x)\|^2$.

However, most fundamental and applied research conducted in systems and control regarding Lyapunov stability theory deals exclusively with 
\emph{continuous-time} systems. Unfortunately, (dynamical) stability properties are generally not preserved for simple forward-Euler, and sample-and-hold discretizations~\cite{Stuart1998}. Furthermore, practical implementations of optimization algorithms in modern digital computers demand discrete-time algorithms of these continuous-time systems. 

\section{Preliminaries: Optimization Algorithms as Discrete-Time Systems}

\label{sec:optalgsgenreal}
Generalizing \eqref{eq:gradientflow}, \eqref{eq:qRGF} and \eqref{eq:qSGF}, consider a continuous-time algorithm (dynamical system) modeled via an 
ordinary differential equation (ODE) 
\begin{equation}
    \dot{x} = F(x)
    \label{eq:generalflow}
\end{equation}
for $t\geq 0$, or, more generally, a differential inclusion
\begin{equation}
    \dot{x}(t) \in \mathcal{F}(x(t)),\ \text{a.e. } t\geq 0,
    \label{eq:generalDI}
\end{equation}
such that $x(t)\to x^\star$ as $t\to t^\star$. For $q$-RGF~(\ref{eq:qRGF}) and $q$-SGF~(\ref{eq:qSGF}) with gradient dominated functions, we have finite-time convergence, and thus $t^\star = t^\star(x(0)) < \infty$.

Most common optimization schemes can be written in a state-space form as:
\begin{equation}
\label{eq:general_numericalmethod}
    X_{k+1} = F_\mathrm{d}(k,X_k),\quad
    x_k = G(X_k),\quad
    k\in\mathbb{Z}_+\triangleq \{0,1,2,\ldots\}
\end{equation}
for a given $X_0\in\mathbb{R}^m$ (typically $m\geq n$), where $F_\mathrm{d}:\mathbb{Z}_+\times\mathbb{R}^m\to\mathbb{R}^m$ and $G:\mathbb{R}^m\to\mathbb{R}^n$. 
Naturally,~(\ref{eq:general_numericalmethod}) can be seen as a discrete-time dynamical system constructed by discretizing~(\ref{eq:generalflow}) in time. In particular, we have $x_k\approx x(t_k)$, where $\{0 = t_0 < t_1 < t_2 < \ldots\}$ denotes a time partition and $x(\cdot)$ a solution to~(\ref{eq:generalflow}) or~(\ref{eq:generalDI}) as appropriate.  Therefore, we call $X_k$ and $x_k$, respectively, the \emph{state} and \emph{output} at time step $k$. 

\begin{example}
The standard gradient descent (GD) algorithm
\begin{equation*}
    x_{k+1} = x_k - \eta\nabla f(x_k)
    \label{eq:GD}
\end{equation*}
with stepsize (learning rate) $\eta > 0$ can be readily written in the form~(\ref{eq:general_numericalmethod}) by taking $m=n$, $F_\mathrm{d}(x) \triangleq x - \eta\nabla f(x)$, and $G(x) \triangleq x$. 
\begin{itemize}
    \item If the stepsizes are adaptive, \emph{i.e.} if we replace $\eta$ by a sequence $\{\eta_k\}$ with $\eta_k>0$, then we only need to replace $F_\mathrm{d}(k,x) \triangleq x - \eta_k\nabla f(x)$, provided that $\{\eta_k\}$ is not computed using feedback from~$\{x_k\}$ (\emph{e.g.} through line search). 
    \item  If we do wish to use time-varying step-size, then we can set $m=n+1$, $G([x;\eta]) \triangleq x$, and $F_\mathrm{d}([x;\eta]) \triangleq [F_\mathrm{d}^{(1)}([x;\eta]); F_\mathrm{d}^{(2)}([x;\eta])]$, where $F_\mathrm{d}^{(1)}([x;\eta]) \triangleq x - \eta\nabla f(x)$, and $F_\mathrm{d}^{(2)}$ is a user-defined function that dictates the updates in the stepsize. In particular, an open-loop adaptive stepsize $\{\eta_k\}$ may be achieved under this scenario, provided that it is possible to write $\eta_{k+1} = F_\mathrm{d}^{(2)}(\eta_k)$. 
    \item If we wish to use individual stepsizes for each the $n$ components of $\{x_k\}$, then it suffices to take $\eta_k$ as an $n$-dimensional vector (thus $m=2n$), and make appropriate changes in $F_\mathrm{d}$ and $G$.
\end{itemize}

In each of these cases, GD can be seen as a forward-Euler discretization of the GF~(\ref{eq:gradientflow}), \emph{i.e.}, let the adaptive time step $\Delta t_k \triangleq t_{k+1} - t_k$ chosen as $\Delta t_k = \eta_k$, and $x_{k+1} = x_k + \Delta t_k F_\mathrm{GF}(x_k)$.
\end{example}

\begin{example}
The proximal point algorithm (PPA)
\begin{equation*}
    x_{k+1} = \argmin_{x\in\mathbb{R}^n}\left\{f(x) + \frac{1}{2\eta_k}\|x - x_k\|_2^2\right\}
    \label{eq:PPA}
\end{equation*}
with stepsize $\eta_k>0$ (open loop, for simplicity) can also be written in the form~(\ref{eq:general_numericalmethod}), by taking $m=n$, $F_\mathrm{d}(k,x) \triangleq \argmin_{x'\in\mathbb{R}^n}\{f(x') + \frac{1}{2\eta_k}\|x' - x\|_2^2\}$, and $G(x) \triangleq x$. Naturally, we need to assume sufficient regularity for $F_\mathrm{d}(k,x)$ to exist and we must design a consistent way to choose $F_\mathrm{d}(k,x)$ when $F_\mathrm{d}$ attains multiple minimizers. Alternatively, these conditions must be satisfied, at the very least, at every  $(k,x)\in\{(0,x_0),(1,x_1),(2,x_2),\ldots\}$ for a particular chosen initial $x_0\in\mathbb{R}^n$. 
By assuming sufficient regularity, we have $\nabla\{f(x) + \frac{1}{2\eta_k}\|x - x_k\|_2^2\}\rvert_{x=x_{k+1}} = 0$, and thus
\begin{equation*}
    \begin{split}
        \nabla f(x_{k+1}) + \frac{1}{\eta_k}(x_{k+1}-x_k) = 0
        \quad\iff\quad 
        x_{k+1} = x_k +\Delta t_k F_\mathrm{GF}(x_{k+1})
    \end{split}
\end{equation*}
with $\Delta t_k = \eta_k$, which is precisely the backward-Euler discretization of the GF~(\ref{eq:gradientflow}).
\end{example}

\section{Analysis of Euler Discretization of RGF/SGF}
Following the previous discussion on continuous-time gradient flow, we turn now to the algorithmic perspective, in the discrete-time setting. We propose the following general algorithm framework based on the forward Euler discretization of the flows:
\begin{equation}
        x_{k+1} = x_k + \eta F(x_k),\;\eta>0
    \label{eq:euler}
\end{equation}
where $F\in\{F_{q-\mathrm{RGF}},F_{q-\mathrm{SGF}}\}$. We will show later that this discretization leads, for small enough $\eta>0$, to solutions that are close to the solutions of the continuous-time flows. Also with a little abuse of notations, we still use RGF/SGF to call the discrete-time algorithms corresponding to the RGF/SGF flow discretizations.

\subsection{Continuous-Time Convergence of $q$-RGF/SGF}
Here
we review the necessary conditions to ensure the finite-time convergence of RGF/SGF. 
Here the parameter $c$ in \eqref{eq:qRGF} and \eqref{eq:qSGF} will not be explicitly denoted in $F_{q-\mathrm{RGF}},F_{q-\mathrm{SGF}}$. Next, borrowing terminologies in~\cite{wilson2019}, we introduce the main assumption in the work.

\begin{assumption}[Gradient Dominance]
    \label{assume:gradient_dominance}
    We assume the function $f$ to be continuously differentiable, and $\mu$-\emph{gradient dominated of order} $p\in (1,\infty]$, i.e.,
    \begin{equation*}
        \frac{p-1}{p}\|\nabla f(x)\|_2^{\frac{p}{p-1}}\geq \mu^{\frac{1}{p-1}}\autopar{f(x) - f(x^\star)},
        \quad
        \forall\ x\in\mathbb{R}^n,
    \end{equation*}
    where $\mu>0$ and $x^\star\in\argmin_{x\in\mathbb{R}^n}f(x)$ is the local minimizer of $f$, also we denote the optimal value $f^\star\triangleq f(x^\star)$. 
\end{assumption}

\begin{remark}
It is well-known continuously differentiable strongly-convex functions are gradient dominated of order $p=2$.
Also, gradient dominance with $p=2$ is known as the \textit{Polyak-{\L}ojasiewicz~(PL) condition}, which was introduced by~\cite{Polyak1963} to relax the common strong convexity assumption.
Furthermore, if $f$ is gradient dominated (of any order) w.r.t. $x^\star$, then $x^\star$ is an isolated stationary point of $f$.

The generalized notion of the gradient dominance above is strongly tied to the {\L}ojasiewicz inequality which is guaranteed for analytic functions~\cite{Lojasiewicz1963,Lojasiewicz1965,Polyak1963,Lojasiewicz1999,Bolte2007}, it was later further extended to \textit{Kurdyka-{\L}ojasiewicz (KL) inequality} \cite{kurdyka1998gradients,AB09,attouch2010proximal} to fit more general cases (e.g., nonsmoothness).
More precisely, {\L}ojasiewicz inequality is typically written as: for some $C > 0$ and $\theta\in \left(\frac{1}{2},1\right]$, $\|\nabla f(x)\|_2 \geq C \cdot |f(x) - f^\star|^{\theta}$
holds for every $x\in\mathbb{R}^n$ in a small enough open neighborhood of the stationary point $x=x^\star$. 
So it is easy to see that analytic functions are always gradient dominated. However, while analytic functions are always smooth, smoothness is not required to attain gradient dominance. 
Also this condition is shown to be valid in several real machine learning problems, for example in reinforcement learning, the value functions with softmax parameterization meet the above condition \cite{mei2020global,mei2021leveraging}, which further rationalizes our settings.
\end{remark}

The following theorem in \cite{ourICMLpaperanon} summarized the finite-time convergence results of $q$-RGF~\eqref{eq:qRGF} and $q$-SGF~\eqref{eq:qSGF} in the continuous-time sense.

\begin{theorem}[\cite{ourICMLpaperanon}]
    \label{thm1}
    With Assumption~\ref{assume:gradient_dominance},
    let $c>0$ and $q\in (p,\infty]$. Then, any maximal solution $x(\cdot)$, in the sense of Filippov, to the $q$-\textnormal{RGF}~(\ref{eq:qRGF}) or $q$-\textnormal{SGF}~(\ref{eq:qSGF}) will converge in finite time to $x^\star$, provided that $\|x(0)-x^\star\|_2>0$ is sufficiently small. More precisely, $\displaystyle\lim _{t\to t^\star} x(t) = x^\star$, where the convergence time $t^\star <\infty$ may depend on which flow is used, but in both cases is upper bounded by 
    \begin{equation*}
        t^\star \leq \frac{\|\nabla f(x_0)\|_2^{\frac{1}{\theta}-\frac{1}{\theta'}}}{cC^\frac{1}{\theta}\left(1-\frac{\theta}{\theta'}\right)},
        \quad
        \text{where }
        C = \left(\frac{p}{p-1}\right)^{\frac{p-1}{p}}\mu^{\frac{1}{p}},\ 
        \theta = \frac{p-1}{p},\ 
        \theta' = \frac{q-1}{q}.
        \label{eq:tstarboundgeneralizedCF1}
    \end{equation*}
    In particular, given any compact and positively invariant subset $S\subset\mathcal{D}$, both flows converge in finite time with the aforementioned convergence time upper bound (which can be tightened by replacing $\overline{\mathcal{D}}$ with $S$) for any $x_0\in S$. Furthermore, if $\mathcal{D}=\mathbb{R}^n$, then we have global finite-time convergence, \emph{i.e.} finite-time convergence to any maximal solution (in the sense of Filippov) $x(\cdot)$ with arbitrary $x_0\in\mathbb{R}^n$.
    \label{thm:firstorderflows}
\end{theorem}
Here the analysis leverages the gradient dominance to show the energy function $\mathcal{E}(t)\triangleq f(x(t)) - f^\star$ satisfies the Lyapunov-like differential inequality $\dot{\mathcal{E}}(t)=\mathcal{O}(\mathcal{E}(t)^\alpha)$ for some $\alpha < 1$. The detailed proof is recalled in Appendix \ref{section-theorem1-proof} for completeness.

\subsection{Closeness of the Discretization}
Now we turn to the proximity analysis, aiming to show the closeness between solutions of the continuous flows and the trajectories of their forward Euler discretization. 
 
\begin{theorem}[Closeness]
    \label{theorem2}
    Suppose $f$ is continuously differentiable, locally $L_f$-Lipschitz, and $\mu$-gradient dominated of order $p\in (1,\infty)$ in a compact neighborhood $S$ of a strict local minimizer~$x^\star\in\mathbb{R}^n$ ($\mu, L_f>0$). Let $c>0$ and $q\in (p,\infty]$. Then, for a given initial condition $x_0\in S$, any maximal solution $x(t),\;x(0)=x_0$, (in the Filippov sense) to $q$-\textnormal{RGF} or $q$-\textnormal{SGF}, there exists an arbitrarily small $\epsilon>0$ such that the solution $x_k$ of the discrete algorithm (\ref{eq:euler}), with sufficiently small $\eta>0$, are $\epsilon$-close to $x(t)$, i.e., $\|x_k-x(t)\|_2\leq \epsilon$ for $|t-k\eta|<\epsilon$, and we have:
    \begin{equation}
        \label{eq:bound1}
        \begin{split}
            f(x_k)-f(x^\star)\leq
            \begin{cases}
                L_f\epsilon+\automedpar{\autopar{f(x_0)-f(x^\star)}^{(1-\alpha)}-\tilde{c}(1-\alpha)\eta k}^{(1-\alpha)} & k\leq k^{\star} \\
                L_f \epsilon & k>k^{\star}
            \end{cases}
        \end{split}
    \end{equation}
    where 
    $\alpha=\frac{\theta}{\theta^{'}},\;
    \theta=\frac{p-1}{p},\;
    \theta^{'}= \frac{q-1}{q},
    $ 
    and
    $$
    \tilde{c}=c\left(\left(\frac{p}{p-1}\right)^{\frac{p-1}{p}}\mu^{\frac{1}{p}}\right)^{\frac{1}{\theta^{'}}},\quad
    k^{\star}=\frac{(f(x_0)-f(x^\star))^{(1-\alpha)}}{\tilde{c}(1-\alpha)\eta}.
    $$
\end{theorem}

To prove Theorem \ref{theorem2}, we borrow some tools and results from hybrid control systems theory\footnote{Note that there might be several ways of approaching this proof. For instance, one could follow the general results on stochastic approximation of set-valued dynamical systems, using the notion of perturbed solutions to differential inclusions presented in \cite{BHS05}.}, which is characterized by continuous flows with discrete jumps between the continuous flows. They are often modeled by differential inclusions added to discrete mappings to model the jumps between the differential inclusions. The optimization flows proposed here can be regarded as a simple case of hybrid systems with one differential inclusion and a possible jump or discontinuity at the optimum. With that we use the tools and results of \cite{ST10}, which study how a certain class of hybrid systems behave after discretization with a certain class of discretization schemes. In other words, \cite{ST10} quantifies, under some conditions, how close are the solutions of the discretized hybrid dynamical system to the solutions of the original hybrid system. 

Here we denote the differential inclusion of the continuous optimization flow by $\mathcal{F}:\mathbb{R}^n\rightrightarrows\mathbb{R}^n$, and its discretization in time by $\mathcal{F}_\mathrm{d}:\mathbb{R}^n\rightrightarrows\mathbb{R}^n$.
We first recall a definition, which we adapt from the general case of jumps between multiple differential inclusions~\cite[Definition 3.2]{ST10} to our case of one differential inclusion or flow.

\begin{definition}[$(T,\epsilon)$-closeness] 
Given $T>0,\;\epsilon>0,\;\eta>0$, two solutions $x_t:\;[0,\;T]\rightarrow\mathbb{R}^n$, and $x_k:\;\{0,1,2,...\}\rightarrow\mathbb{R}^n$ are $(T,\epsilon)$-close if:
\begin{itemize}
    \item[(a)] $\exists~k\in\{1,2,...\}$ such that $|t-k\eta|<\epsilon$, and $\|x_t(t)-x_k (k)\|_2<\epsilon$ for $\forall~t\leq T$.
    \item[(b)] $\exists~t\leq T$ such that $|t-k\eta|<\epsilon$, and $\|x_t(t)-x_k (k)\|_2<\epsilon$ for $\forall~k\in\{1,2,...\}$.
\end{itemize}
\end{definition}
Next, we will recall Theorem 5.2 in \cite{ST10}, while adapting it to our special case of a hybrid system with one differential inclusion\footnote{A set-valued mapping $\mathcal{F}: \mathbb{R}^n\rightrightarrows\mathbb{R}^n$ is \emph{outer semicontinuous} if for each sequence
$\{x_i\}_{i=1}^{\infty}$ converging to a point $x\in \mathbb{R}^n$ and each sequence $y_i\in\mathcal{F}(x_i)$ converging to a point $y$, it holds that $y\in\mathcal{F}(x)$. It is \emph{locally bounded} if, for each $x\in\mathbb{R}^n$, there exists compact sets $K,K'\subset\mathbb{R}^n$ such that $x\in K$ and $\mathcal{F}(K)\triangleq\cup_{x\in K}\mathcal{F}(x)\subset K^{'}$. In what follows, we use the following notations: Given a set $A$, $\textnormal{con} A$ denotes the convex hull, and $\mathbb{B}$ denotes the closed unit ball in a Euclidean space.}.

\begin{lemma}[Closeness of continuous and discrete solutions \citep{ST10}]\label{theorem2_background}
    Consider the differential inclusion~\eqref{eq:generalDI}
    for a given set-valued mapping $\mathcal{F}: \mathbb{R}^m\rightrightarrows\mathbb{R}^m$ assumed to be outer semicontinuous, locally bounded, nonempty, and with convex values for every $x\in\mathcal{C}$, for some closed set $\mathcal{C}\subseteq\mathbb{R}^m$. Consider the discrete-time system $\mathcal{F}_\mathrm{d}: \mathbb{R}^n\rightrightarrows\mathbb{R}^n$, such that, for each compact set $K\subset\mathbb{R}^n$, there exists $\rho\in\mathcal{K}_{\infty}$, and $\eta^\star>0$ such that for each $x\in K$ and each $\eta\in(0,\;\eta^\star]$, 
    \begin{equation}\label{cond_A}
        \mathcal{F}_\mathrm{d}(x)\subset x+\eta\; \textnormal{con}\mathcal{F} (x+\rho(\eta)\mathbb{B})+\eta\rho(\eta)\mathbb{B}.
    \end{equation}
    Then, for a compact set $K\subset\mathbb{R}^n$ and $\forall~\epsilon>0$, time horizon $T\geq 0$, there exists $\eta^\star>0$ such that: $\forall~\eta\in(0,\eta^\star]$ and any discrete solution $x_k$ with $x_k (0)\in K+\delta\mathbb{B},\;\delta>0$, there exists a continuous solution $x_t$ with $x_t (0)\in K$ such that $x_k$ and $x_t$ are $(T,\epsilon)$-close.
\end{lemma}

Now we are ready to prove Theorem~\ref{theorem2}. 

\begin{proof}[Proof of Theorem~\ref{theorem2}]
    First, note that outer semicontinuity follows from the upper semicontinuity and the closedness of the Filippov differential inclusion map. Furthermore, local boundedness follows from continuity everywhere outside stationary points, which are isolated.
    Now, let us examine their discretization by forward-Euler. The mapping $\mathcal{F}_\mathrm{d}$ in this case is a singleton, given by
    \begin{equation*}
        \mathcal{F}_\mathrm{d} (x) \triangleq x + \eta F(x),
    \label{eq:euler_proof}
    \end{equation*}
    where $\eta > 0$, which clearly satisfies condition~(\ref{cond_A}).
    With Lemma \ref{theorem2_background} we conclude about the $(T,\epsilon)$-closeness between the continuous-time solutions of the flows  $\mathcal{F}$ ($q$-RGF (\ref{eq:qRGF}), $q$-SGF (\ref{eq:qSGF})), and the discrete-time solutions. 
    
    Finally, using the Lyapunov function $V=f-f(x^\star)$ as defined in the proof of Theorem \ref{thm1}, together with inequalities (\ref{ineedyou}), (\ref{eq:Vdotinesol}), and a local Lipschitz bound on $f$, one can derive the weak convergence bound given by (\ref{eq:bound1}), as follows: with $L_{f }>0,\;\epsilon>0,$
    \begin{equation}\label{proof_Eq}
        \begin{array}{l}
        \|f(x_k)-f(x^\star)-(f(x_t)-f(x^\star))\|_2=\|f (x_k)- f(x_t) \|_2\leq \tilde{\epsilon}=L_{f }\epsilon,\;\\
        \|f(x_k)-f(x^\star)\|_2-\|(f(x_t)-f(x^\star))\|_2\leq \|f (x_k)- f(x_t) \|_2\leq \tilde{\epsilon},\\
        \|f(x_k)-f(x^\star)\|_2\leq  \tilde{\epsilon}+\|f(x_t)-f(x^\star)\|_2,\\
        \|f(x_k)-f(x^\star)\|_2\leq \tilde{\epsilon}+[(f(x_0)-f(x^\star))^{(1-\alpha)}-\tilde{c}(1-\alpha)\eta k]^{1/(1-\alpha)}, \textnormal{for} \;k\leq k^{\star},
        \end{array}
    \end{equation}
    where $\alpha=\frac{\theta}{\theta^{'}},\ \theta=\frac{p-1}{\psi},\ \theta^{'}= \frac{q-1}{q},$, and 
    $$\tilde{c}=c\autopar{\autopar{\frac{\psi}{p-1}}^{\frac{p-1}{\psi}}\mu^{\frac{1}{\psi}}}^{\frac{1}{\theta^{'}}},\ k^{\star}=\frac{(f(x_0)-f(x^\star))^{(1-\alpha)}}{\tilde{c}(1-\alpha)\eta}.$$
    
    Next, the case for $k>k^\star$ is rather straightforward: Indeed, for $t>t^\star$ by finite-time convergence result, we have $x_t=x^\star$, which directly leads to the bound $  \|f(x_k)-f(x^\star)\|_2\leq L_f \epsilon,\;k>k^{\star}$, since the term $\|(f(x_t)-f(x^\star))\|_2$ in (\ref{proof_Eq}) vanishes.
    \qed
\end{proof}

Theorem \ref{theorem2} shows that the $\epsilon$-convergence of $x_k\to x^\star$ can be achieved in a finite number of steps upper bounded by $k^\star=\frac{(f(x_0)-f(x^\star))^{(1-\alpha)}}{\tilde{c}(1-\alpha)\eta}$. This is a preliminary convergence result aiming to show the existence of discrete solutions obtained via the discretization (\ref{eq:euler}), which are $\epsilon$-close to the continuous solutions of the finite-time flows. We also underline here that after $x_k$ reaches an $\epsilon$-neighborhood of $x^\star$, then $x_{k+1} \approx  x_k,\;\forall k>k^\star$, since $x^\star$ is an equilibrium point of the continuous flows, i.e., $\nabla f(x^\star)=0$, e.g. see Definition \ref{equi_point} in Appendix \ref{subsec:FTSsemiS}.

\section{Convergence Rate Analysis}
\label{sec:convergence_discrete}
In this section, we will further study the convergence rate of our proposed algorithm \eqref{eq:qRGF}, \eqref{eq:qSGF}, and \eqref{eq:euler}.
The result in the previous section provides the relationship between the continuous flows and their forward Euler discretizations, but it cannot provide practical instructions on the choice of stepsize $\eta$.

To derive convergence rates of the proposed discretization with practical stepsize $\eta$, we need to introduce an extra assumption of $L$-Lipschitz smooth, as presented below.

\begin{assumption}[Lipschitz Smoothness of Order $q$]
    \label{assume:Holder_smooth}
    We assume the function $f$ is $L$-Lipschitz smooth of order $q\in(1,2]$, i.e., for any $x,y \in\mathbb{R}^n$,
    \begin{equation*}
        \autonorm{\nabla f(y)-\nabla f(x)}_2\leq L\autonorm{y-x}_2^{q-1}.
    \end{equation*}
\end{assumption}

\begin{remark}
    The assumption is also called as $(L, q)$-H\"older continuity \cite{devolder2014first,nesterov2015universal,yashtini2016global}, which is a generalization of the common Lipschitz smoothness condition ($q=2$). It is easy to show that the condition will lead to the following property \cite{nesterov2015universal}: $f(y)\leq f(x)+\autoprod{\nabla f(x), y-x}+\frac{L}{q}\autonorm{y-x}_2^q$,
    which is also called \emph{weak smoothness}, and has been extensively studied in many modern machine learning applications~\cite{zhang2022convergence}. 
\end{remark}

Now we start our discussion on the convergence rate. The main result is presented in the following theorem.

\begin{theorem}[Convergence Rate]
    \label{theorem3}
    Suppose the function $f$ satisfies Assumption \ref{assume:gradient_dominance} and \ref{assume:Holder_smooth}, by running $q$-RGF/SGF~\eqref{eq:euler}, we have
    \begin{itemize}
        \item For $q$-RGF with a stepsize $\eta=L^{\frac{1}{1-q}}$, we have
        \begin{equation}\label{eq:RGF_Deter_result}
            f(x_K)-f^\star\leq\autopar{1-\kappa^{\frac{1}{1-q}}}^K\autopar{f(x_0)-f^\star},
        \end{equation}
        \item For $q$-SGF with a stepsize with $\eta=(n^{\frac{q}{2}}L)^{\frac{1}{1-q}}$, we have
        \begin{equation}\label{eq:SGF_Deter_result}
            f(x_K)-f^\star\leq\autopar{1-(n^\frac{q}{2}\kappa)^{\frac{1}{1-q}}}^K\autopar{f(x_0)-f^\star},
        \end{equation}
    \end{itemize}    
    where $n$ is the dimension number and $\kappa\triangleq\frac{L}{\mu}$. 

\end{theorem}

For the proof, we divide the discussion based on different flows\footnote{Note that here for convenience, we fix the order in gradient dominance as $q$, while previous discussion assumes that $q>p$.}.

\begin{proof}[Proof of $q$-RGF]
    Following the definitions of smoothness, we have
    \begin{equation*}
        \begin{split}
            f(x_{k+1})-f^\star
            \leq\ &
            f(x_k)-f^\star+\autoprod{\nabla f(x_k),x_{k+1}-x_{k}}+\frac{L}{q}\autonorm{x_{k+1}-x_k}_2^q\\
            =\ &
            f(x_k)-f^\star-\eta\frac{\autonorm{\nabla f(x_k)}_2^2}{\autonorm{\nabla f(x_k)}_2^{\frac{q-2}{q-1}}}+\frac{L\eta^q}{q}\frac{\autonorm{\nabla f(x_k)}_2^q}{\autonorm{\nabla f(x_k)}_2^{\frac{q(q-2)}{q-1}}}\\
            \overset{\eta=L^{\frac{1}{1-q}}}{=}\ &
            f(x_k)-f^\star-\frac{q-1}{q}L^{\frac{1}{1-q}}\autonorm{\nabla f(x_k)}_2^{\frac{q}{q-1}}\\
            \leq\ &
            \autopar{1-\kappa^{\frac{1}{1-q}}}\autopar{f(x_k)-f^\star},
        \end{split}
    \end{equation*}
    which verifies the conclusion by recursion. 
    \qed
\end{proof}

\begin{proof}[Proof of $q$-SGF]
    Similarly, for the $q$-SGF case, we have
    \begin{equation*}
        \begin{split}
            f(x_{k+1})-f^\star
            \leq\ &
            f(x_k)-f^\star+\autoprod{\nabla f(x_k),x_{k+1}-x_{k}}+\frac{L}{q}\autonorm{x_{k+1}-x_k}_2^q\\
            =\ &
            f(x_k)-f^\star-\eta\autonorm{\nabla f(x_k)}_1^{\frac{1}{q-1}}\autoprod{\nabla f(x_k),\sign\autopar{\nabla f(x_k)}}+\frac{L\eta^q n^\frac{q}{2}}{q}\autonorm{\nabla f(x_k)}_1^{\frac{q}{q-1}}\\
            =\ &
            f(x_k)-f^\star-\autopar{\eta-\frac{n^\frac{q}{2}L\eta^q}{q}}\autonorm{\nabla f(x_k)}_1^{\frac{q}{q-1}}\\
            \leq\ &
            f(x_k)-f^\star-\frac{q-1}{q}(n^\frac{q}{2}L)^{\frac{1}{1-q}}\autonorm{\nabla f(x_k)}_2^{\frac{q}{q-1}}\\
            \leq\ &
            \autopar{1-(n^\frac{q}{2}\kappa)^{\frac{1}{1-q}}}\autopar{f(x_k)-f^\star},
        \end{split}
    \end{equation*}
    Here the last equality follows the setting of $\eta$ and the fact that $\autonorm{x}_2\leq\autonorm{x}_1$ for any $x\in\mathbb{R}^n$. Finally we verify the conclusion by recursion.
    \qed
\end{proof}

\begin{remark}
    The above results show that with a constant stepsize, both $q$-RGF and $q$-SGF converge to the optimal point with a linear rate. It is easy to find that the corresponding iteration complexities are $\mathcal{O}(\kappa^{\frac{1}{q-1}}\ln\frac{1}{\epsilon}) $ and $\mathcal{O}((n^{\frac{q}{2}}\kappa)^{\frac{1}{q-1}}\ln\frac{1}{\epsilon})$, respectively.
    Note that for the classical Lipschitz smooth case (i.e. $q=2$), the RHS of the RGF result \eqref{eq:RGF_Deter_result} will be reduced to
    $$(1-\kappa^{-1})^K\autopar{f(x_0)-f^\star},$$
    which is the same as \cite[Theorem 1]{Karimi2016} and \cite[Theorem 11]{attouch2010proximal}; while for SGF \eqref{eq:SGF_Deter_result}, it will be
    $$(1-(n\kappa)^{-1})^K\autopar{f(x_0)-f^\star},$$
    which recovers the result in \cite[Theorem 13]{beznosikov2020biased}\footnote{Note that $\frac{1}{n}\,\|\nabla f(x)\|_1^{-1}\sign(\nabla f(x))$ is an $\frac{1}{n}$-approximate compressor of $\nabla f(x)$ as mentioned in \cite{karimireddy2019error}, so we can apply \cite[Theorem 13]{beznosikov2020biased} to recover the result.}. So we can conclude that our results extend the classical results in the Lipschitz smoothness case.
\end{remark}

\subsection{Convergence Rates in The Stochastic Case}
Attaining the full gradient may be expensive in practical applications due to the large data size, where stochastic optimization algorithms are more applicable~\cite{robbins1951stochastic,bottou2018optimization}. So here we further the discussion into the case where the function $f$ can be written as 
\begin{equation}
    \label{eq:stochastoc_objective}
    f(x)\triangleq\EE_\xi[f_\xi(x)],
\end{equation}
where $\xi\sim\Xi$ is a random variable following some unknown distribution $\Xi$, and we have only access to finite samples of $\Xi$. To formalize the discussion, we impose the following assumption on the gradient estimator.
\begin{assumption}[Gradient Estimator]
    \label{defn:unbiased_gradient}
    For any given $x\in\mathbb{R}^n$ and $\xi\sim\Xi$, we assume $\mathbb{E}_\xi\automedpar{\nabla f_\xi(x)}=\nabla f(x)$ and $\mathbb{E}_\xi\autonorm{\nabla f_\xi(x)-\nabla f(x)}_2^2\leq \sigma^2$.
\end{assumption}

The above assumption on the gradient estimator is very common in stochastic optimization literature \cite{ghadimi2013stochastic,bottou2018optimization}. Next we discuss the stochastic counterpart of $q$-RGF/SGF algorithm, where we use $\nabla f_\xi$ to build the gradient estimator, also we call the corresponding algorithms stochastic $q$-RGF/SGF.

\subsubsection{Stochastic $q$-RGF}
We present the convergence result of stochastic $q$-RGF in the following theorem. For convenience, we denote $\psi\triangleq\frac{q}{q-1}\in[2,+\infty]$.

\begin{theorem}[Convergence Rate of Stochastic RGF]
    \label{theorem4}
    Suppose that the function $f$ and $f_\xi$ in \eqref{eq:stochastoc_objective} satisfies Assumption \ref{assume:gradient_dominance}, \ref{assume:Holder_smooth} and \ref{defn:unbiased_gradient}, if we run 
    $q$-RGF (\eqref{eq:qRGF} and \eqref{eq:euler}) while replacing the full gradient by its unbiased estimator $g(x)\triangleq\frac{1}{B}\sum_{i=1}^{B}f_{\xi_i}(x)$ with batch size $B\in\mathbb{N}^+$ and $\{\xi_1, \xi_2,\cdots,\xi_B\}$ are randomly drawn samples from $\Xi$. Then for stochastic $q$-RGF if we set $\eta = (\psi L)^{\frac{1}{1-q}}, B=(\frac{2\cdot\autopar{2\sigma}^\psi\mu^{\frac{1}{1-q}}}{\epsilon})^{\frac{2}{\psi}},$ we have
    \begin{equation*}
        \begin{split}
            \mathbb{E}\automedpar{f(x_K)-f^\star}
            \leq
            \autopar{1-\frac{2}{(2\psi)^\psi}\kappa^{\frac{1}{1-q}}}^K\autopar{f(x_0)-f^\star}+\frac{\epsilon}{2}.
        \end{split}
    \end{equation*}
\end{theorem}

\begin{proof}
    Following the definitions of smoothness, we have
    \begin{equation*}
        \begin{split}
            f(x_{k+1})
            \leq\ &
            f(x_k)+\autoprod{\nabla f(x_k),x_{k+1}-x_{k}}+\frac{L}{q}\autonorm{x_{k+1}-x_k}_2^q\\
            =\ &
            f(x_k)-\eta\autoprod{\nabla f(x_k)-g(x_k)+g(x_k),\frac{g(x_k)}{\autonorm{g(x_k)}_2^{\frac{q-2}{q-1}}}}+\frac{L\eta^q}{q}\frac{\autonorm{g(x_k)}_2^q}{\autonorm{g(x_k)}_2^{\frac{q-2}{q-1}}}\\
            =\ &
            f(x_k)-\autopar{\eta-\frac{L\eta^q}{q}}\autonorm{g(x_k)}_2^{\psi}-\eta\autoprod{\nabla f(x_k)-g(x_k),\frac{g(x_k)}{\autonorm{g(x_k)}_2^{\frac{q-2}{q-1}}}}\\
            \leq\ &
            f(x_k)-\autopar{\eta-\frac{L\eta^q}{q}}\autonorm{g(x_k)}_2^{\psi}+\eta\autopar{\frac{1}{q}\autonorm{\frac{g(x_k)}{\autonorm{g(x_k)}_2^{\frac{q-2}{q-1}}}}_2^q+\frac{1}{\psi}\autonorm{\nabla f(x_k)-g(x_k)}_2^{\psi}}\\
            =\ &
            f(x_k)-\autopar{\frac{\eta}{\psi}-\frac{L\eta^q}{q}}\autonorm{g(x_k)}_2^{\psi}+\frac{\eta}{\psi}\autonorm{\nabla f(x_k)-g(x_k)}_2^{\psi},
        \end{split}
    \end{equation*}
    where the second inequality comes from the Young's inequality. Then note that by Jensen's inequality,
    \begin{equation*}
        \label{eq:Jensen_SFE-RGF}
        \begin{split}
            \autonorm{\nabla f(x_k)}_2^{\psi}
            =
            2^{\psi}\autonorm{\frac{1}{2}\autopar{\nabla f(x_k)-g(x_k)+g(x_k)}}_2^{\psi}
            \leq\ &
            2^{\psi-1}\autopar{\autonorm{g(x_k)}_2^{\psi}+\autonorm{\nabla f(x_k)-g(x_k)}_2^{\psi}},
        \end{split}
    \end{equation*}
   which induces 
   $$-\autonorm{g(x_k)}_2^{\psi}\leq-2^{1-\psi}\autonorm{\nabla f(x_k)}_2^{\psi}+\autonorm{\nabla f(x_k)-g(x_k)}_2^{\psi}.$$
   Next, note that the setting of $\eta$ will ensure the coefficient of $\autonorm{g(x_k)}_2^{\psi}$ to be negative, so we have
    \begin{equation*}
        \begin{split}
            f(x_{k+1})
            \leq\ &
            f(x_k)-\autopar{\frac{\eta}{\psi}-\frac{L\eta^q}{q}}\autopar{2^{1-\psi}\autonorm{\nabla f(x_k)}_2^{\psi}-\autonorm{\nabla f(x_k)-g(x_k)}_2^{\psi}}
            +\frac{\eta}{\psi}\autonorm{\nabla f(x_k)-g(x_k)}_2^{\psi}\\
            =\ &
            f(x_k)-2^{1-\psi}\autopar{\frac{\eta}{\psi}-\frac{L\eta^q}{q}}\autonorm{\nabla f(x_k)}_2^{\psi}+\autopar{\frac{2\eta}{\psi}-\frac{L\eta^q}{q}}\autonorm{\nabla f(x_k)-g(x_k)}_2^{\psi},
        \end{split}
    \end{equation*}
    then by taking the expectation on both sides and subtracting both sides by $f^\star$, with the bounded variance and gradient dominance, we have
    \begin{equation*}
        \begin{split}
            \mathbb{E}\automedpar{f(x_{k+1})-f^\star}
            \leq\ &
            \mathbb{E}\automedpar{\autopar{1-2^{1-\psi}\autopar{\frac{\eta}{\psi}-\frac{L\eta^q}{q}}p\mu^{\frac{1}{q-1}}}\autopar{f(x_k)-f^\star}+\autopar{\frac{2\eta}{\psi}-\frac{L\eta^q}{q}}\autopar{\frac{\sigma^2}{B}}^\frac{\psi}{2}},
        \end{split}
    \end{equation*}
    then set $\eta=(pL)^{\frac{1}{1-q}}$, $B=b^2=(\frac{2\cdot\autopar{2\sigma}^{\psi}\mu^{\frac{1}{1-q}}}{\epsilon})^{\frac{2}{\psi}}$ and
    $K=\lceil\frac{\autopar{2\psi}^{\psi}}{2\kappa^{\frac{1}{1-q}}}\log\autopar{\frac{2\Delta}{\epsilon}}\rceil,$
    we have
    \begin{equation*}
        \begin{split}
            \mathbb{E}\automedpar{f(x_K)-f^\star}
            \leq\ &
            \mathbb{E}\automedpar{\autopar{1-\frac{2}{(2\psi)^{\psi}}\kappa^{\frac{1}{1-q}}}\autopar{f(x_k)-f^\star}+\frac{2\sigma^{\psi}}{\psi^{\psi}b^{\psi}}L^{\frac{1}{1-q}}}\\
            \leq\ &
            \mathbb{E}\automedpar{\autopar{1-\frac{2}{(2\psi)^{\psi}}\kappa^{\frac{1}{1-q}}}^K\Delta+\frac{2\sigma^{\psi}}{\psi^{\psi}b^{\psi}}L^{\frac{1}{1-q}}\sum_{i=0}^{K-1}\autopar{1-\frac{2}{(2\psi)^{\psi}}\kappa^{\frac{1}{1-q}}}^i}\\
            \leq\ &
            \mathbb{E}\automedpar{\autopar{1-\frac{2}{(2\psi)^{\psi}}\kappa^{\frac{1}{1-q}}}^K\Delta+\frac{2\sigma^{\psi}}{\psi^{\psi}b^{\psi}}L^{\frac{1}{1-q}}\cdot\frac{(2\psi)^{\psi}}{2}\kappa^{\frac{1}{q-1}}}\\
            \leq\ &
            \mathbb{E}\automedpar{\exp{\autopar{-\frac{2\kappa^{\frac{1}{1-q}}}{(2\psi)^{\psi}}K}}\Delta+\frac{(2\sigma)^{\psi}\mu^{\frac{1}{1-q}}}{b^{\psi}}}
            \leq
            \frac{\epsilon}{2}+\frac{\epsilon}{2}=\epsilon,
        \end{split}
    \end{equation*}
    while the total sample complexity is (note that $\frac{1}{\psi}+\frac{1}{q}=1$, $q\in(1,2]$, which implies $\psi\in[2,+\infty)$)
    \begin{equation*}
        \begin{split}
            K\cdot b^2
            \geq\ &
            \frac{\autopar{2\psi}^{\psi}}{2\kappa^{\frac{1}{1-q}}}\log\autopar{\frac{2\Delta}{\epsilon}}\cdot \autopar{\frac{2\cdot\autopar{2\sigma}^{\psi}\mu^{\frac{1}{1-q}}}{\epsilon}}^{\frac{2}{\psi}}\\
            =\ &
            \mathcal{O}\autopar{\sigma^2\kappa^{\frac{1}{q-1}}\mu^{-\frac{2}{q}}\psi^{\psi}\epsilon^{-\frac{2}{\psi}}\log\frac{\Delta}{\epsilon}},
        \end{split}
    \end{equation*}
    which verifies the conclusion.
    \qed
\end{proof}

\begin{remark}
Here we set the batch size $B$ as large as $\mathcal{O}(\epsilon^{-2/\psi})$. We argue that such batch size with dependence on $\epsilon$ is common in stochastic optimization literature, e.g., \cite{ghadimi2016stochastic,Fangetal2018}. In the literature, there are several algorithms that use only a single sample. In such case, the result above indicates that we can still get a linear convergence to the neighborhood of the optimal point, whereas applying (non-trivial) mini-batch drives the update to the optimal point precisely.
Also note that existing results resorting to a single sample generally require additional assumptions like bounded gradient norm or the use of a decaying stepsize, e.g., \cite[Theorem 4]{Karimi2016}. Currently, there seem to be some challenges preventing us from extending our analysis to the single sample setting without enforcing extra conditions, we leave this problem to future work.
\end{remark}

\subsubsection{Stochastic $q$-SGF}
The analysis of $q$-SGF in the stochastic case is a bit more complicated. Indeed, note that previously our discussions are based on $\ell_2$-norm, here however, following the argument in \cite{balles2020geometry}, we change the gradient dominance and Lipschitz smoothness assumption to hold in $\ell_\infty$-norm (and its dual norm $\ell_1$). 
\begin{assumption}[$\ell_\infty$-Gradient Dominance and Lipschitz Smoothness]
    \label{assume:Stoc_SGF_Modified_Assumption}
    We assume the function $f$ is $\mu$-gradient dominated of order $q$ and $L$-Lipschitz smooth of order $q$ under $\ell_\infty$-norm (and its dual norm $\ell_1$), i.e., for any $x,y \in\mathbb{R}^n$, we have
    \begin{equation*}
        \begin{split}
            \frac{q-1}{q}\|\nabla f(x)\|_1^{\frac{q}{q-1}}
            \geq 
            \mu^{\frac{1}{q-1}}(f(x) - f^\star),
            \quad
            \autonorm{\nabla f(y)-\nabla f(x)}_1
            \leq
            L\autonorm{y-x}_\infty^{q-1}.
        \end{split}
    \end{equation*}
\end{assumption}
Similarly the new smoothness condition above implies that 
$$
f(y)\leq f(x)+\autoprod{\nabla f(x), y-x}+\frac{L}{q}\autonorm{y-x}_\infty^q,\quad \forall\,x,y \in\mathbb{R}^n.
$$ 
And for the gradient estimator, we need the following assumption.
\begin{assumption}[SPB Estimator \cite{safaryan2021stochastic}]
    \label{eq:Stoc_SGF_unbiased_gradient_estimator}
    For given $x\in\mathbb{R}^n$ and $\xi\in\Xi$, we assume
    $$
    \mathbb{E}_\xi\automedpar{\nabla f_\xi(x)}=\nabla f(x),\quad
    \mathbb{E}_\xi\autonorm{\nabla f_\xi(x)-\nabla f(x)}_1^2\leq \sigma^2.
    $$
    We further assume that there exists $B^*\in\mathbb{N}^+$ such that for any $B\geq B^*$, the (mini-batch) gradient estimator $g(x)\triangleq\frac{1}{B}\sum_{i=1}^B\nabla f_{\xi_i}(x)$, where $\{\xi_1, \xi_2,\cdots,\xi_B\}$ are randomly drawn samples from $\Xi$, satisfies the \emph{success probability bounds (SPB) property}, i.e., for each component of $g(x)$ and $\nabla f(x)$ (denoted as $g_i(x)$ and $\nabla_i f(x)$), $\exists\,p^*\in(\frac{1}{2},1]$ such that for any $i\in\autobigpar{1,2,\cdots,n}$, we have
    \begin{equation*}
        \mathbb{P}\autopar{\sign\autopar{\nabla_i f(x)}=\sign\autopar{g_i(x)}}\geq p^*>\frac{1}{2}.
    \end{equation*}
\end{assumption}

Note that similar assumptions under $\ell_\infty$-norm (and its dual norm $\ell_1$) are widely adapted in the literature of sign-based algorithms, e.g., \cite{Karimi2016,bernstein2018signsgd,bernstein2018signsgd2,balles2020geometry}. Here the extra SPB assumption, proposed in \cite{safaryan2021stochastic}, is very important in the convergence analysis of stochastic sign-based algorithms.
With the above preparation, we summarize the result of stochastic SGF in the following theorem.

\begin{theorem}[Convergence Rate of Stochastic SGF]
    \label{theorem5}
    Suppose the function $f$ satisfies the setting in Assumption \ref{assume:Stoc_SGF_Modified_Assumption}. Consider $q$-SGF (\eqref{eq:qSGF} and \eqref{eq:euler}) with gradient estimator $g(x)$ with a batch size $B$ satisfying Assumption \ref{eq:Stoc_SGF_unbiased_gradient_estimator}. If we set 
    $\eta=(\frac{L2^{\psi-1}}{2p^*-1})^{\frac{1}{1-q}},\ B=\max\{(\frac{2\mu^{\frac{1}{1-q}}}{q\epsilon})^{\frac{2}{\psi}}\autopar{2p^*-1}^{-4}\sigma^2, B^*\},$
    then we have
    \begin{equation*}
        \begin{split}
            \mathbb{E}\automedpar{f(x_K)-f^\star}
            \leq
            \autopar{1-\autopar{2p^*-1}^{\psi}\autopar{\kappa 2^{\psi-1}}^{\frac{1}{1-q}}}^K\autopar{f(x_0)-f^\star}+\frac{\epsilon}{2}.
        \end{split}
    \end{equation*}
\end{theorem}

\begin{proof}
    Following Assumption \ref{assume:Stoc_SGF_Modified_Assumption} and \ref{eq:Stoc_SGF_unbiased_gradient_estimator}, we have
    \begin{equation*}
        \begin{split}
            f(x_{k+1})
            \leq\ &
            f(x_k)+\autoprod{\nabla f(x_k),x_{k+1}-x_{k}}+\frac{L}{q}\autonorm{x_{k+1}-x_k}_\infty^q\\
            =\ &
            f(x_k)-\eta\autonorm{g(x_k)}_1^{\frac{1}{q-1}}\autoprod{\nabla f(x_k),\sign\autopar{g(x_k)}}+\frac{L\eta^q}{q}\autonorm{g(x_k)}_1^{\frac{q}{q-1}},
        \end{split}
    \end{equation*}
    conditioning on $x_k$, we have
    \begin{equation*}
        \label{eq:Stoc_SGF_conditioning}
        \begin{split}
            \mathbb{E}\automedpar{f(x_{k+1})\,|\,x_k}
            \leq
            f(x_k)-\mathbb{E}\automedpar{\eta\autonorm{g(x_k)}_1^{\frac{1}{q-1}}\autoprod{\nabla f(x_k),\sign\autopar{g(x_k)}}\,|\,x_k}+\mathbb{E}\automedpar{\frac{L\eta^q}{q}\autonorm{g(x_k)}_1^{\frac{q}{q-1}}\,|\,x_k},
        \end{split}
    \end{equation*}
    note that for each component of $g(x_k)$, we have
    \begin{equation*}
        \begin{split}
            &\mathbb{E}\automedpar{\autonorm{g(x_k)}_1^{\frac{1}{q-1}}\sign\autopar{g_i(x_k)}\,|\,x_k}\\
            =\ &
            \mathbb{E}\Bigg[\autonorm{g(x_k)}_1^{\frac{1}{q-1}}\mathbb{P}\autopar{\sign\autopar{\nabla_i f(x_k)}=\sign\autopar{g_i(x_k)}}\sign\autopar{\nabla_i f(x_k)}\\
            &\qquad
            -
            \autonorm{g(x_k)}_1^{\frac{1}{q-1}}\mathbb{P}\autopar{\sign\autopar{\nabla_i f(x_k)}\neq\sign\autopar{g_i(x_k)}}\sign\autopar{\nabla_i f(x_k)}
            \,|\,x_k\Bigg]\\
            =\ &
            \mathbb{E}\automedpar{\autonorm{g(x_k)}_1^{\frac{1}{q-1}}\autopar{2\mathbb{P}\autopar{\sign\autopar{\nabla_i f(x_k)}=\sign\autopar{g_i(x_k)}}-1}\sign\autopar{\nabla_i f(x_k)}\,|\,x_k},
        \end{split}
    \end{equation*}
    then with Assumption \ref{eq:Stoc_SGF_unbiased_gradient_estimator} and the setting that $B\geq B^*$, we have the property
    $\mathbb{P}\autopar{\sign\autopar{\nabla_i f(x_k)}=\sign\autopar{g_i(x_k)}}\geq p^*>\frac{1}{2},$
    which implies
    \begin{equation*}
        \begin{split}
            &\mathbb{E}\automedpar{f(x_{k+1})\,|\,x_k}\\
            \leq\ &
            f(x_k)-\mathbb{E}\automedpar{\eta\autonorm{g(x_k)}_1^{\frac{1}{q-1}}\autopar{2p^*-1}\autoprod{\sign\autopar{\nabla f(x_k)},\nabla f(x_k)}
            +\frac{L\eta^q}{q}\autonorm{g(x_k)}_1^{\frac{q}{q-1}}\,|\,x_k}\\
            \leq\ &
            f(x_k)-\mathbb{E}\automedpar{\eta\autopar{2p^*-1}\autonorm{g(x_k)}_1^{\frac{1}{q-1}}\autonorm{\nabla f(x_k)}_1\,|\,x_k}+\mathbb{E}\automedpar{\frac{L\eta^q}{q}\autonorm{g(x_k)}_1^{\frac{q}{q-1}}\,|\,x_k}\\
            \leq\ &
            f(x_k)-\eta\autopar{2p^*-1}\autonorm{\nabla f(x_k)}_1^{\frac{q}{q-1}}+\frac{L\eta^q}{q}\mathbb{E}\automedpar{\autonorm{g(x_k)}_1^{\frac{q}{q-1}}\,|\,x_k}
        \end{split}
    \end{equation*}
    where the last inequality comes from the convexity of $\autonorm{\cdot}_1^{\frac{1}{q-1}}$. Then for the last term in the RHS, by decomposition, we have
    \begin{equation*}
        \begin{split}
            \autonorm{g(x_k)}_1^{\psi}
            =\ &
            2^{\psi}\autonorm{\frac{\nabla g(x_k)-f(x_k)+f(x_k)}{2}}_1^{\psi}
            \leq
            2^{\psi-1}(\autonorm{f(x_k)}_1^{\psi}+\autonorm{\nabla f(x_k)-g(x_k)}_1^{\psi})
        \end{split}
    \end{equation*}
    so we have (recall that $\psi=\frac{q}{q-1}$)
    \begin{equation*}
        \label{eq:Stoc_SGF_Before_Telescope}
        \begin{split}
            &\mathbb{E}\automedpar{f(x_{k+1})\,|\,x_k}\\
            \leq\ &
            f(x_k)-\eta\autopar{2p^*-1}\autonorm{\nabla f(x_k)}_1^{\frac{q}{q-1}}\\
            &\qquad\qquad\qquad
            +\frac{L\eta^q2^{\psi-1}}{q}\mathbb{E}\automedpar{\autonorm{f(x_k)}_1^{\psi}+\autonorm{\nabla f(x_k)-g(x_k)}_1^{\psi}\,|\,x_k}\\
            \leq\ &
            f(x_k)-\autopar{\autopar{2p^*-1}\eta-\frac{L\eta^q2^{\psi-1}}{q}}\autonorm{\nabla f(x_k)}_1^{\frac{q}{q-1}}+\frac{L\eta^q2^{\psi-1}}{q}\autopar{\frac{\sigma^2}{B}}^{\frac{\psi}{2}}\\
            \leq\ &
            f(x_k)-\autopar{\autopar{2p^*-1}\eta-\frac{L\eta^q2^{\psi-1}}{q}}\frac{q\mu^{\frac{1}{q-1}}}{q-1}\autopar{f(x_k)-f^\star}+\frac{L\eta^q2^{\psi-1}}{q}\autopar{\frac{\sigma^2}{B}}^{\frac{\psi}{2}},
        \end{split}
    \end{equation*}
    so take $\eta\equiv(\frac{L2^{\psi-1}}{2p^*-1})^{\frac{1}{1-q}}$ and
    $$
    B\geq b^2=\autopar{\frac{2\mu^{\frac{1}{1-q}}}{q}}^{\frac{2}{\psi}}\autopar{2p^*-1}^{-4}\sigma^2\epsilon^{-\frac{2}{\psi}},
    \quad
    K=\autoceil{\autopar{2p^*-1}^{-\psi}\autopar{\kappa 2^{\psi-1}}^{\frac{1}{q-1}}\ln{\frac{2\Delta}{\epsilon}}},
    $$
    we have
    \begin{equation*}
        \begin{split}
            &\mathbb{E}\automedpar{f(x_K)-f^\star}\\
            \leq\ &
            \mathbb{E}\automedpar{\autopar{1-\frac{1}{\psi}\autopar{2p^*-1}^{\psi}\autopar{L2^{\psi-1}}^{\frac{1}{1-q}}\frac{q\mu^{\frac{1}{q-1}}}{q-1}}\autopar{f(x_{K-1})-f^\star}+\frac{\autopar{L2^{\psi-1}}^{\frac{1}{1-q}}}{q\autopar{2p^*-1}^{\psi}}\autopar{\frac{\sigma^2}{b^2}}^{\frac{\psi}{2}}}\\
            =\ &
            \mathbb{E}\automedpar{\autopar{1-\autopar{2p^*-1}^{\psi}\autopar{\kappa 2^{\psi-1}}^{\frac{1}{1-q}}}\autopar{f(x_{K-1})-f^\star}+\frac{\autopar{L2^{\psi-1}}^{\frac{1}{1-q}}}{q\autopar{2p^*-1}^{\psi}}\autopar{\frac{\sigma^2}{b^2}}^{\frac{\psi}{2}}}\\
            \leq\ &
            \autopar{1-\autopar{2p^*-1}^{\psi}\autopar{\kappa 2^{\psi-1}}^{\frac{1}{1-q}}}^K\autopar{f(x_0)-f^\star}\\
            &\qquad\qquad\qquad\qquad
            +\frac{\autopar{L2^{\psi-1}}^{\frac{1}{1-q}}}{q\autopar{2p^*-1}^{\psi}}\autopar{\frac{\sigma^2}{b^2}}^{\frac{\psi}{2}}\cdot\frac{1}{\autopar{2p^*-1}^{\psi}\autopar{\kappa 2^{\psi-1}}^{\frac{1}{1-q}}}\\
            \leq\ &
            \exp\autopar{-\autopar{2p^*-1}^{\psi}\autopar{\kappa 2^{\psi-1}}^{\frac{1}{1-q}}K}\Delta+\frac{\autopar{\mu}^{\frac{1}{1-q}}}{q\autopar{2p^*-1}^{2\psi}}\autopar{\frac{\sigma^2}{b^2}}^{\frac{\psi}{2}}\\
            =\ &
            \frac{\epsilon}{2}+\frac{\epsilon}{2}
            =
            \epsilon,
        \end{split}
    \end{equation*}
    which concludes the proof.
    \qed
\end{proof}

\begin{remark}
    Theorem~\ref{theorem4} and \ref{theorem5} show that stochastic RGF/SGF drives the function value to linearly converge to the $\mathcal{O}(\epsilon)$-neighborhood of the optimal value. As a comparison, \cite{lei2019stochastic} studied the convergence guarantee of stochastic gradient descent (SGD), which can be viewed as stochastic $2$-RGF, under a similar setting. Also for stochastic SGF, we notice that recently \cite{li2021faster} provided a similar convergence result of stochastic SGF in the classical setting (i.e., $q=2$), their result also requires a lower bound on the probability of estimation correctness, which corresponds to the SPB assumption above. Our result can be viewed as an extension of their work.
\end{remark}

\section{Extensions to Non-Uniform Smoothness and Gradient Dominance}
The parameter settings above often rely on the knowledge of the objective function, e.g., the constant smoothness parameter (in Assumption~\ref{assume:gradient_dominance} and \ref{assume:Holder_smooth}), which may be hard to attain in practice. One remedy to the issue is the line-search strategy. For example, consider $q$-RGF~\eqref{eq:euler}, referring to the result in Theorem \ref{theorem3}, if we have no access to the value of the H\"older continuity parameter $(L, q)$, we can extend the idea in \cite{nesterov2015universal} and add an extra line-search step based on a modification of the Armijo rule before updating the variable. Similar line-search strategy idea also appeared in existing optimization literature, e.g., \cite{jin2021high}.

\subsection{Problem Setting and Algorithm Design}
To formalize the discussion, we still consider the unconstrained minimization problem for a given cost function $f$. But different from the settings in Section~\ref{sec:convergence_discrete}, here we assume the following:
\begin{assumption}[Non-Uniform Smoothness and Gradient Dominance]
    \label{assume:varying_Q}
    We assume the function $f$ satisfies that 
    \begin{itemize}
        \item The function $f$ is continuously differentiable.
        \item The function $f$ is $\autopar{\mu(x), q}$-gradient dominated near a strict local minimizer $x^*\in\mathbb{R}^n$ where $q\in(1,2]$, i.e., for any $x\in\mathbb{R}^n$,
        \begin{equation*}
            \frac{q-1}{q}\autonorm{\nabla f(x)}_2^{\frac{q}{q-1}}\geq\mu(x)^{\frac{1}{q-1}}\autopar{f(x)-f^*}.
        \end{equation*}
        \item The function $f$ is $\autopar{S(x), q(x)}$-Lipschitz smooth where $q(x)\in(1,2]$, i.e., for any $x,y \in\mathbb{R}^n$,
        \begin{equation*}
            f(y)\leq f(x)+\autoprod{\nabla f(x), y-x}+\frac{S(x)}{q(x)}\autonorm{y-x}_2^{q(x)}.
        \end{equation*}
    \end{itemize}
\end{assumption}
Here the assumption extends the setting previously in Assumption~\ref{assume:gradient_dominance} and \ref{assume:Holder_smooth}, and allows the modulus of smoothness and gradient dominance to be both non-uniform. As a comparison, such a setting also appeared in \cite{mei2021leveraging}. Here we further incorporate the line-search strategy with our proposed RGF/SGF, resulting the Varying-RGF/SGF algorithm, which is presented in Algorithm~\ref{alg:Line-Search-RGF-SGF}. 
\begin{algorithm}[ht]
    \caption{Varying-RGF/SGF with Line-Search (Vary-RGF/SGF-LS)}
    \label{alg:Line-Search-RGF-SGF}
    \begin{algorithmic}[1]
        \REQUIRE Initial point $x_0$, exponent $q$, initial (small enough) guess $L_0$, accuracy $\epsilon$
        \FORALL{$k = 0,1,..., K$}
            \STATE Attain the gradient $\nabla f(x_k)$
            \STATE Find the smallest integer $i_k\geq 0$ such that for
            \begin{equation*}
                x_k^+\triangleq 
                \begin{cases}
                    x_k-(q^{i_k}L_k)^{\frac{1}{1-q}}\nabla f(x_k)/\|\nabla f(x_k)\|_2^{\frac{q-2}{q-1}}
                    & \text{if using RGF}\\
                    x_k-\autopar{n^{\frac{q}{2}}q^{i_k}L_k}^{\frac{1}{1-q}}\autonorm{\nabla f(x_k)}_1^{\frac{1}{q-1}}\sign\autopar{\nabla f(x_k)} & \text{if using SGF}
                \end{cases}
            \end{equation*}
            it satisfies that
            \begin{equation}
                \label{eq:LS_condition}
                f(x_k^+)\leq f(x_k)+\autoprod{\nabla f(x_k),x_k^+-x_k}+\frac{q^{i_k}L_k}{q}\autonorm{x_k^+-x_k}_2^q+\frac{\epsilon}{2}
            \end{equation}
            \STATE Update: $x_{k+1}=x_k^+$, $L_{k+1}=q^{i_k}L_k$
        \ENDFOR
    \end{algorithmic}
\end{algorithm}

\subsection{Convergence Analysis}
To proceed with the analysis, according to \cite{nesterov2015universal,diakonikolas2021complementary}, we know that we only need to consider a varying smoothness parameter.
First, the following result from \cite{nesterov2015universal} is important in the analysis.
\begin{lemma}[\cite{nesterov2015universal}]
\label{lm:holder_exponent_change}
    If a function $f$ is $(S,p)$-Lipschitz smooth with $p\in(1,2]$, then for any $q\in(1,2]$, $\delta>0$ and $L\geq [\frac{2(q-p)}{pq\delta}]^{\frac{q-p}{p}}S^{\frac{q}{p}}$,
    we have for any $x$ and $y$
    \begin{equation*}
        f(y)\leq f(x)+\autoprod{\nabla f(x), y-x}+\frac{L}{q}\autonorm{y-x}_2^q+\frac{\delta}{2}
    \end{equation*}
\end{lemma}

The above result ensures the existence of $i_k$ in \eqref{eq:LS_condition}. We can then fix the exponent $q\in(1,2]$ and focus on the search of the smoothness coefficient $L(x,q,\delta)$. Without loss of generality, we always assume the initial guess $L_0$ is small enough. The main convergence result for RGF is summarized below.
\begin{theorem}[Vary-RGF-LS]
    \label{thm:vary-RGF-LS}
    Suppose that $f$ satisfies Assumption \ref{assume:varying_Q}, define
    \begin{equation*}
        L(x_k, q, \Tilde{\epsilon})\triangleq\automedpar{\frac{2(q-q(x_k))}{q(x_k)q\Tilde{\epsilon}}}^{\frac{q-q(x_k)}{q(x_k)}}S(x_k)^{\frac{q}{q(x_k)}},
        \quad
        \Tilde{\epsilon}\triangleq\kappa^{\frac{1}{1-q}}\epsilon,
    \end{equation*}
    where
    $\kappa\triangleq\underset{0\leq k\leq K}{\max} \frac{L_{k+1}}{\mu(x_k)}$,
    by running Algorithm~\ref{alg:Line-Search-RGF-SGF} with RGF, we have
    \begin{equation*}
        f(x_K)-f^*\leq (1-\kappa^{\frac{1}{1-q}})^K\autopar{f(x_0)-f^*}+\frac{\epsilon}{2},
    \end{equation*}
    to attain an $\epsilon$-optimal solution, the iteration complexity is $\mathcal{O}(\kappa^{\frac{1}{q-1}}\ln\frac{1}{\epsilon})$.
\end{theorem}

\begin{proof}
    First note that Lemma~\ref{lm:holder_exponent_change} implies that for any $x$, we have
    \begin{equation*}
        f(x)\leq f(x_k)+\autoprod{\nabla f(x_k), x-x_k}+\frac{L(x_k, q, \Tilde{\epsilon})}{q}\autonorm{x-x_k}_2^q+\frac{\Tilde{\epsilon}}{2},
    \end{equation*}
    then note that $L_{k+1}=q^{i_k}L_k\leq L(x_k, q, \Tilde{\epsilon})\leq L$, and by the definition of $x_k^+$, we have
    \begin{equation*}
        x_{k+1}-x_k=-\autopar{L_{k+1}}^{\frac{1}{1-q}}\nabla f(x_k)/\autonorm{\nabla f(x_k)}_2^{\frac{q-2}{q-1}},
    \end{equation*}
    and
    \begin{equation*}
        \begin{split}
            f(x_{k+1})-f^*
            \leq\ &
            f(x_k)-f^*-\autopar{L_{k+1}}^{\frac{1}{1-q}}\autonorm{\nabla f(x_k)}_2^\frac{q}{q-1}+\frac{L_{k+1}}{q}\autopar{L_{k+1}}^{\frac{q}{1-q}}\autonorm{\nabla f(x_k)}_2^\frac{q}{q-1}+\frac{\Tilde{\epsilon}}{2}
            \\
            \leq\ & 
            f(x_k)-f^*-\autopar{\autopar{L_{k+1}}^{\frac{1}{1-q}}-\autopar{L_{k+1}}^{\frac{1}{1-q}}/q}\autonorm{\nabla f(x_k)}_2^\frac{q}{q-1}+\frac{\Tilde{\epsilon}}{2}\\
            \leq\ &
            f(x_k)-f^*-\autopar{\frac{\mu(x_k)}{L_{k+1}}}^{\frac{1}{q-1}}\autopar{f(x_k)-f^*}+\frac{\Tilde{\epsilon}}{2}\\
            \leq\ &
            \autopar{1-\kappa^{\frac{1}{1-q}}}\autopar{f(x_k)-f^*}+\frac{\Tilde{\epsilon}}{2},
        \end{split}
    \end{equation*}
    by induction, we have
    \begin{equation}
        \label{eq:Vary-RGF-recursion}
        \begin{split}
            f(x_K)-f^*
            \leq\ &
            \autopar{1-\kappa^{\frac{1}{1-q}}}^K\autopar{f(x_0)-f^*}+\frac{\Tilde{\epsilon}}{2}\cdot\sum_{i=0}^K\autopar{1-\kappa^{\frac{1}{1-q}}}^i\\
            \leq\ &
            \autopar{1-\kappa^{\frac{1}{1-q}}}^K\autopar{f(x_0)-f^*}+\frac{\Tilde{\epsilon}}{2}\cdot\frac{1}{1-\autopar{1-\kappa^{\frac{1}{1-q}}}}\\
            =\ &
            \autopar{1-\kappa^{\frac{1}{1-q}}}^K\autopar{f(x_0)-f^*}+\frac{\epsilon}{2}
        \end{split}
    \end{equation}
    which concludes the proof, here the complexity is implied by assuming the first term above to be smaller than $\frac{\epsilon}{2}$.
    \qed
\end{proof}

Similarly, we can derive the following result for the SGF algorithm.
\begin{theorem}[Vary-SGF-LS]
    Given the same settings in Theorem \ref{thm:vary-RGF-LS},
    denote
    \begin{equation*}
        L(x_k, q, \Tilde{\epsilon})\triangleq\automedpar{\frac{2(q-q(x_k))}{q(x_k)q\Tilde{\epsilon}}}^{\frac{q-q(x_k)}{q(x_k)}}S^{\frac{q}{q(x_k)}},
        \quad
        \Tilde{\epsilon}\triangleq\autopar{n^{\frac{q}{2}}\kappa}^{\frac{1}{1-q}}\epsilon,
    \end{equation*}
    where
    $\kappa\triangleq\underset{0\leq k\leq K}{\max} \frac{L_{k+1}}{\mu(x_k)}$,
    if we run Algorithm \ref{alg:Line-Search-RGF-SGF} with SGF, we have
    \begin{equation*}
        f(x_K)-f^*\leq(1-(n^{\frac{q}{2}}\kappa)^{\frac{1}{1-q}})^K\autopar{f(x_k)-f^*}+\frac{\epsilon}{2},
    \end{equation*}
    to attain an $\epsilon$-optimal solution, the iteration complexity is $\mathcal{O}((n^{\frac{q}{2}}\kappa)^{\frac{1}{q-1}}\ln\frac{1}{\epsilon})$.
\end{theorem}

\begin{proof}
    First Lemma~\ref{lm:holder_exponent_change} implies that for any $x$, we have
    \begin{equation*}
        \begin{split}
            f(x)
            \leq\ &
            f(x_k)+\autoprod{\nabla f(x_k), x-x_k}+\frac{L(x_k, q, \Tilde{\epsilon})}{q}\autonorm{x-x_k}_2^q+\frac{\Tilde{\epsilon}}{2}\\
            =\ &
            f(x_k)+\autopar{\eta-\frac{L(x_k, q, \Tilde{\epsilon})n^{\frac{q}{2}}\eta^q}{q}}\autonorm{\nabla f(x_k)}_1^{\frac{q}{q-1}}+\frac{\Tilde{\epsilon}}{2}
        \end{split}
    \end{equation*}
    note that $L_{k+1}=q^{i_k}L_k\leq L(x_k, q, \Tilde{\epsilon})\leq L$,
    then by the definition of $x_k^+$ and the setting $\eta=\autopar{n^{\frac{q}{2}}q^{i_k}L_k}^{\frac{1}{1-q}}$, with the fact that $\|x\|_2\leq\|x\|_1$, we have
    \begin{equation*}
        \begin{split}
            f(x_{k+1})-f^*
            \leq\ & 
            f(x_k)-f^*-\frac{q-1}{q}\eta\autonorm{\nabla f(x_k)}_2^\frac{q}{q-1}+\frac{\Tilde{\epsilon}}{2}\\
            \leq\ &
            f(x_k)-f^*-\autopar{\frac{\mu(x_k)}{L_{k+1}n^{\frac{q}{2}}}}^{\frac{1}{q-1}}\autopar{f(x_k)-f^*}+\frac{\Tilde{\epsilon}}{2}\\
            \leq\ &
            \big(1-\big(n^{\frac{q}{2}}\kappa\big)^{\frac{1}{1-q}}\big)\autopar{f(x_k)-f^*}+\frac{\Tilde{\epsilon}}{2}
        \end{split}
    \end{equation*}
    following similar induction in \eqref{eq:Vary-RGF-recursion} and the definition of $\Tilde{\epsilon}$ above, we have
    \begin{equation*}
        \begin{split}
            f(x_K)-f^*
            \leq\ &
            \autopar{1-\autopar{n^{\frac{q}{2}}\kappa}^{\frac{1}{1-q}}}^K\autopar{f(x_0)-f^*}+\frac{\epsilon}{2}
        \end{split}
    \end{equation*}
    which concludes the proof, while the complexity is implied by assuming the first term above to be smaller than $\frac{\epsilon}{2}$.
    \qed
\end{proof}

\section{Numerical Experiments}
\label{sec:num_experiments}
In this section, we use several experiments to verify the effectiveness of our proposed algorithms with fixed parameters.

\subsection{Numerical Experiments on Academic Examples}
Let us show first on a simple numerical example that the acceleration in convergence, proven in continuous time, can translate to convergence acceleration in discrete time. We consider the following function: $f:\mathbb{R}\rightarrow \mathbb{R}$ and 
\begin{equation}    
    \label{eq:numerical_academic}
    f(x)\triangleq \frac{1}{q}\autoabs{x}^{q},
\end{equation}
which is $(q-1)^{q-1}$-gradient dominated of order $q$, and $2$-Lipschitz smooth of order $q$ when $1<q\leq 2$ \cite{lei2020fine,ourICMLpaperanon}. Here we set $q=1.95$ for the objective function. 

The test results of the proposed RGF and SGF algorithms compared with other algorithms is presented in Figure \ref{fig:RGF_GD_toy_result} below. The figure shows both sensitivities to stepsizes and performances of algorithms, measured by the optimality gap $f(x)-f^*$ where $f^*$ is the optimal value which is 
$0$. We can see that $q$-RGF and $q$-SGF, compared to the classical one (2-RGF) and an over-large one (10-RGF), can accommodate larger stepsize and attain better performances with the best-tuned stepsize\footnote{The experiment results show that $q$-RGF and $q$-SGF will drive the function value to exactly 0 with certain stepsizes, which generates plots that may exceed the lower limit in figures.}. Also we further plot the theoretical bounds provided in Theorem \ref{theorem3} (in red), we can see that the proposed algorithms attain linear convergence patterns and their optimal stepsize choice is close to those mentioned in Theorem \ref{theorem3}.

\begin{figure}[ht]
    \centering
    \includegraphics[width=0.47\linewidth]{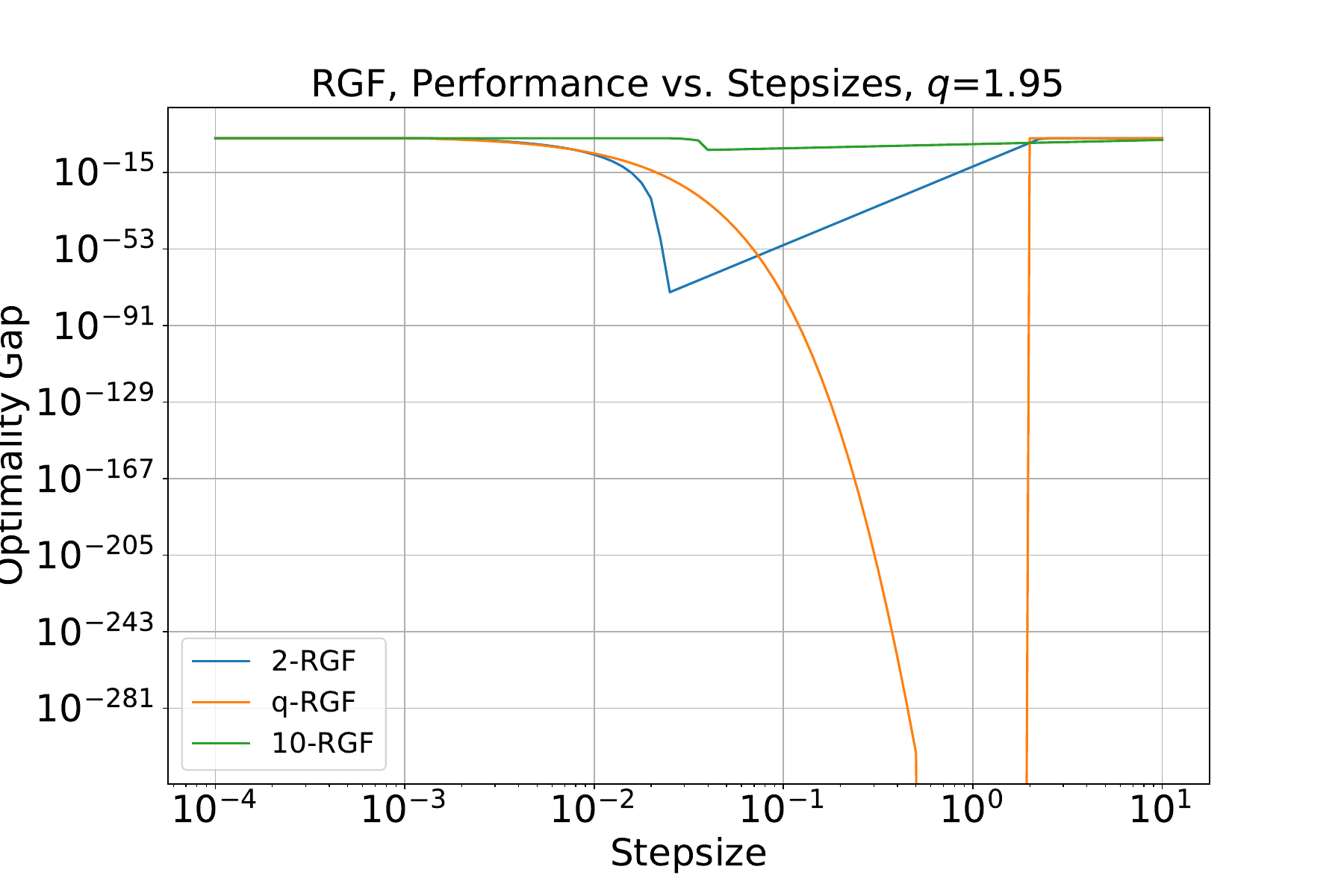}
    \includegraphics[width=0.47\linewidth]{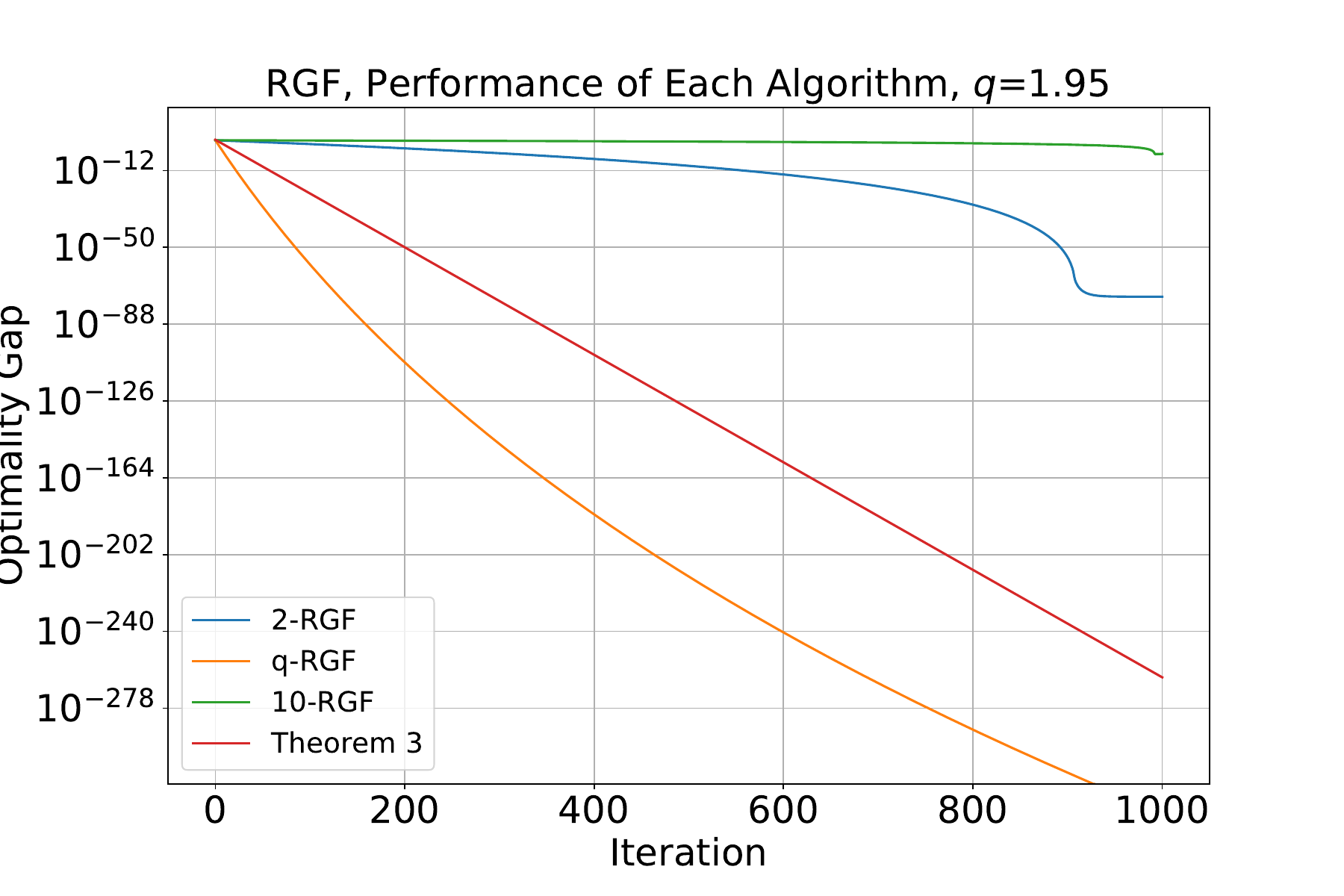}
    \includegraphics[width=0.47\linewidth]{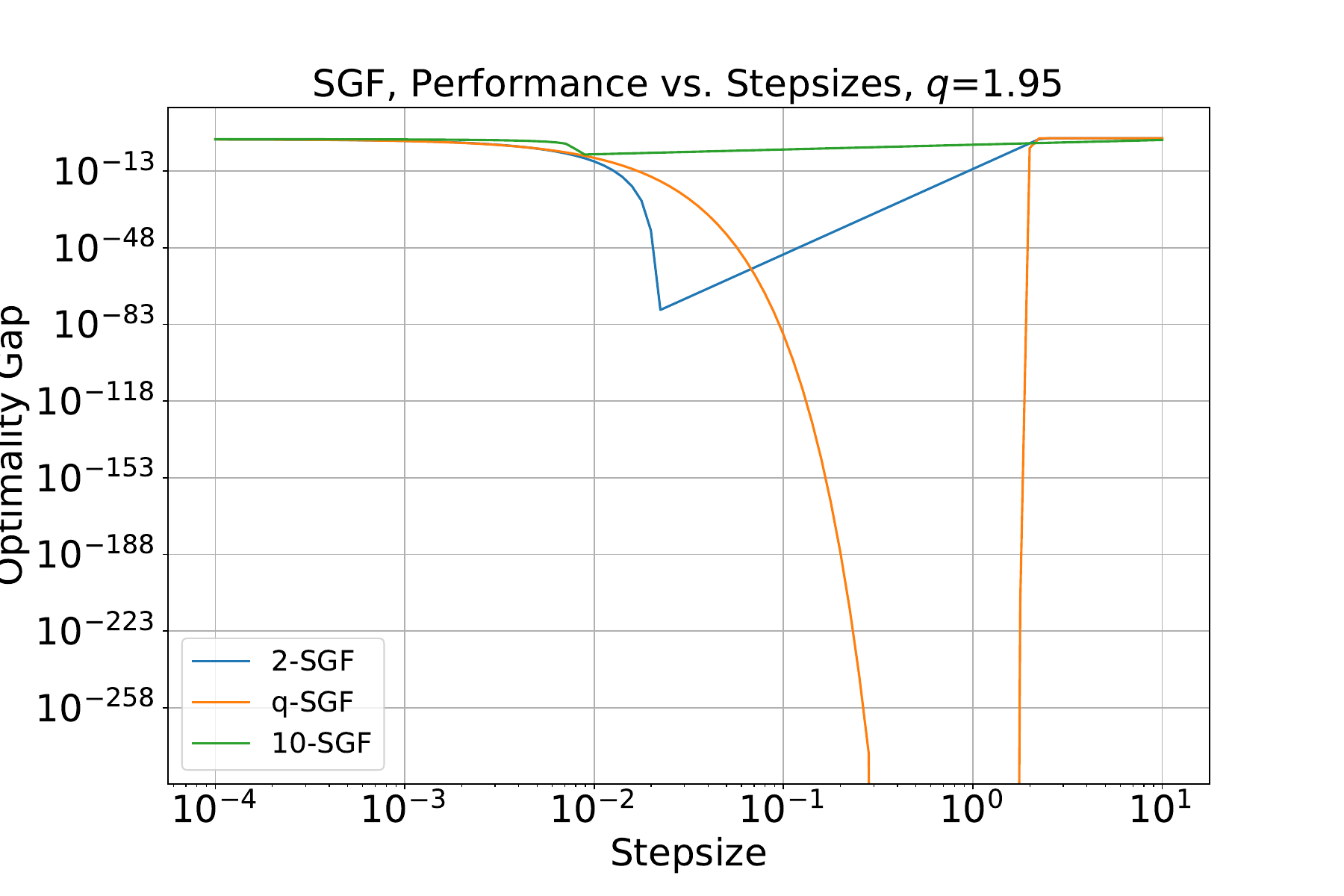}
    \includegraphics[width=0.47\linewidth]{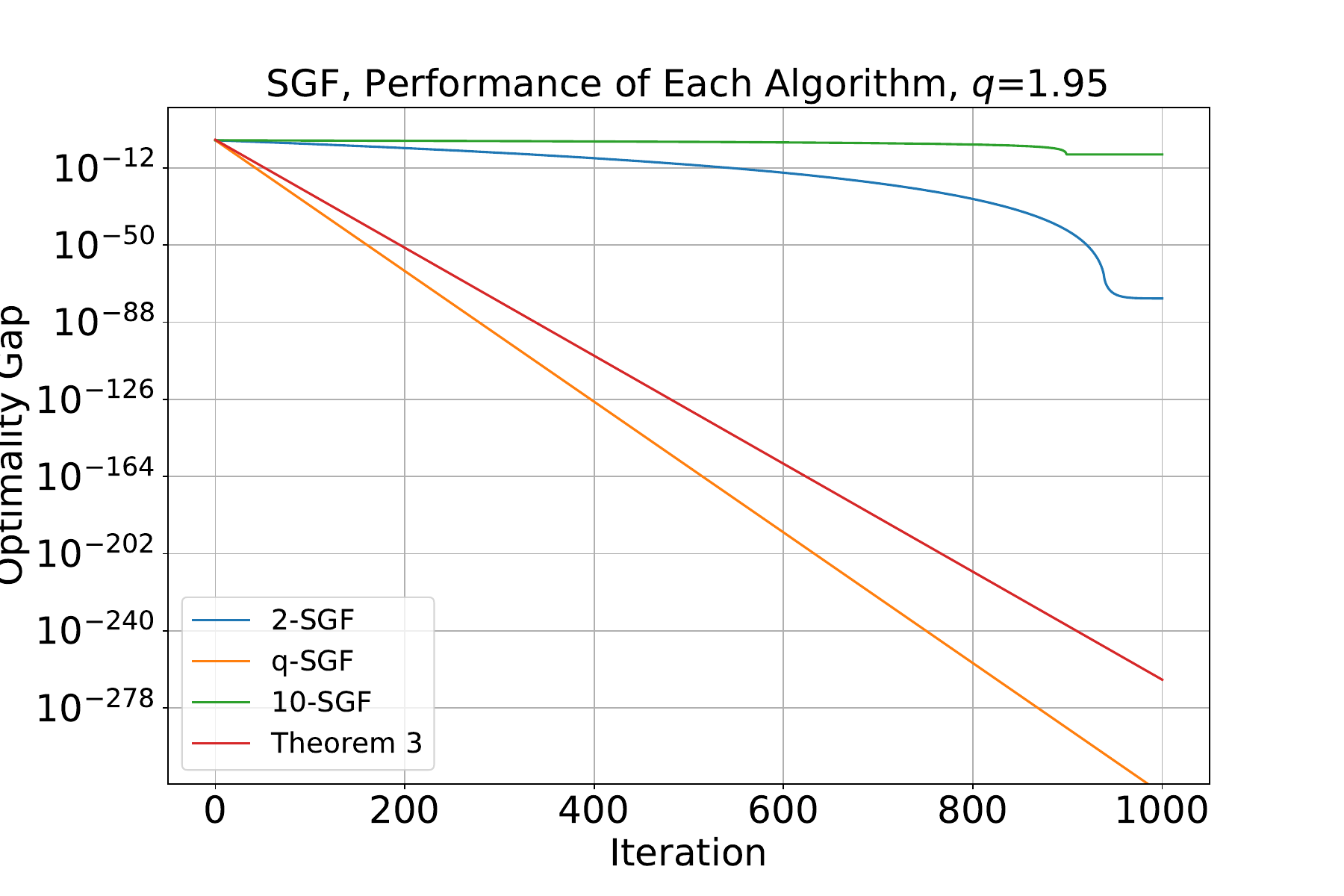}
    \caption{Results of discretization of RGF and SGF on Academic Examples~\eqref{eq:numerical_academic}}
    \label{fig:RGF_GD_toy_result}
\end{figure}

\subsection{Numerical Experiments on Real-world Data}
We report here the results of experiments using deep neural network training on the SVHN dataset\footnote{http://ufldl.stanford.edu/housenumbers/}~\cite{yuval2011reading} in Figure~\ref{fig:gpu-euler-VGG16-SVHN} and Figure~\ref{fig::gpu-euler-VGG16-SVHN-gputime}. Here the experiment is run on a Titan X NVIDIA GPU with 12GB memory, and we use the PyTorch platform to conduct all tests. The results are reported below. 

We tested the performance of Euler discretization of the proposed flows against Adam, and GD algorithms. 
We chose to train the VGG16 CNN model with cross-entropy loss. We divided the dataset into a training set of $74$ batches with $1000$ images each, and a test set of $27$ batches of $1000$ images each, and ran $20$ epochs of training over all the training batches. We tested Euler discretization of $q$-RGF ($c=1,\;q=2.1,\;\eta=0.04\;$), and Euler discretization of $q$-SGF ($c=10^{-3},\;q=2.1,\;\eta=0.04\;$) against GD ($\eta=0.1$) and Adam.
In Figures \ref{fig:gpu-euler-VGG16-SVHN} and \ref{fig::gpu-euler-VGG16-SVHN-gputime}, we can see that both algorithms, Euler $q$-RGF and Euler $q$-SGF, converge faster than GD and Adam for these tests, and reach the same performance on the test set.

\begin{figure}[ht]
    \centering
    \includegraphics[width=0.47\linewidth]{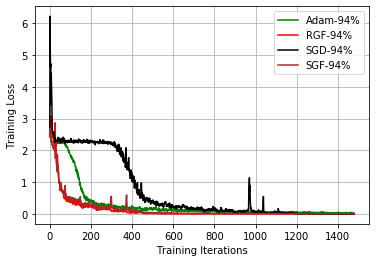}
    \includegraphics[width=0.48\linewidth]{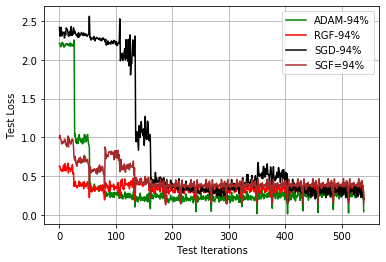}
    \caption{Losses for several optimization algorithms run on the GPU-VGG16-SVHN experiment: Train loss (left), test loss (right)}
    \label{fig:gpu-euler-VGG16-SVHN}
\end{figure}
\begin{figure}[ht]
    \centering
    \includegraphics[width=0.47\linewidth]{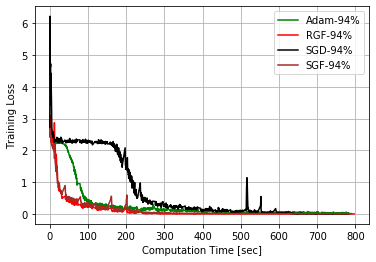}
    \caption{Training loss vs. computation time for the GPU-VGG16-SVHN experiment}
    \label{fig::gpu-euler-VGG16-SVHN-gputime}
\end{figure}

\section{Conclusion}
In this paper, we investigated connections between optimization algorithms and their continuous-time representations (dynamical systems). We reviewed two families of first-order optimization flows for continuous-time optimization, namely the $q$-RGF and $q$-SGF, which are characterized by their finite-time convergence. We then proposed several algorithms based on the forward Euler discretization of RGF/SGF. Using tools from hybrid systems control theory, we provided closeness guarantees of the discrete solutions to their corresponding continuous-time trajectories. Then we discussed the convergence rates of the proposed discrete algorithms in various settings. To the best of our knowledge, this is the first analysis in the literature considering these algorithms in the (structural) nonconvex regime. Finally, we conducted some validation numerical experiments. As future directions, it is interesting to further incorporate the stochastic first-order algorithms in the discretization of RGF and SGF, and study their corresponding performance guarantees; also how to extend the study into more general cases like PL* and KL conditions~\cite{liu2020toward,Lojasiewicz1963}.

\appendix

\section{Discontinuous Systems and Differential Inclusions}
\label{subsec:filippov}
Recall that for an initial value problem (IVP): $x(0) = x_0$ and
\begin{equation}
\dot{x}(t) = F(x(t)) 
\label{eq:IVP}%
\end{equation}
with $F:\mathbb{R}^n\to\mathbb{R}^n$, the typical way to check for existence of solutions is by establishing continuity of $F$. Likewise, to establish uniqueness of the solution, we typically seek Lipschitz continuity. When $F$ is discontinuous, we may understand~(\ref{eq:IVP}) as the Filippov differential inclusion
\begin{equation}
\dot{x}(t) \in \mathcal{K}[F](x(t)),
\label{eq:filippov}
\end{equation}
where $\mathcal{K}[F]:\mathbb{R}^n\rightrightarrows\mathbb{R}^n$ denotes the Filippov set-valued map given by
\begin{equation}
\mathcal{K}[F](x) \triangleq \bigcap_{\delta > 0}\bigcap_{\mu(S)=0}\cobar F(B_\delta(x)\setminus S),
\label{eq:filippovsetvaluedmap}
\end{equation}
where $\mu$ denotes the usual Lebesgue measure and $\cobar$ the convex closure, {i.e.} closure of the convex hull $\co$. For more details, see~\cite{Paden1987}. We can generalize~(\ref{eq:filippov}) to the differential inclusion \cite{Bacciotti1999}
\begin{equation}
\dot{x}(t) \in \mathcal{F}(x(t)),
\label{eq:differentialinclusion}
\end{equation}
where $\mathcal{F}:\mathbb{R}^n\rightrightarrows\mathbb{R}^n$ is some set-valued map.

\begin{definition}[Carath\'{e}odory/Filippov solutions]
We say that $x: [0,\tau)\to\mathbb{R}^n$ with $0<\tau\leq\infty$ is a \mbox{\emph{Carath\'{e}odory solution}} to~(\ref{eq:differentialinclusion}) if $x(\cdot)$ is absolutely continuous and~(\ref{eq:differentialinclusion}) is satisfied a.e. in every compact subset of $[0,\tau)$. Furthermore, we say that $x(\cdot)$ is a \emph{maximal} Carath\'{e}odory solution if no other Carath\'{e}odory solution $x'(\cdot)$ exists with $x = x'|_{[0,\tau)}$. If  $\mathcal{F} = \mathcal{K}[F]$, then Carath\'{e}odory solutions are referred to as \emph{Filippov solutions}.
\end{definition}

For a comprehensive overview of discontinuous systems, including sufficient conditions for existence (Proposition 3) and uniqueness (Propositions 4 and 5) of Filippov solutions, see \cite{Cortes2008}. In particular, it can be established that Filippov solutions to~(\ref{eq:IVP}) exist, provided that the following assumption (Assumption~\ref{ass:existenceFilippov}) holds.

\begin{assumption}[Existence of Filippov solutions]
$F:\mathbb{R}^n\to\mathbb{R}^n$ is defined almost everywhere (a.e.) and is Lebesgue-measurable in a \mbox{non-empty} open neighborhood $U\subset \mathbb{R}^n$ of $x_0\in\mathbb{R}^n$. Further, $F$ is locally essentially bounded in $U$, \emph{i.e.}, for every point $x\in U$, $F$ is bounded a.e. in some bounded neighborhood of $x$.
\label{ass:existenceFilippov}
\end{assumption}

More generally, Carath\'{e}odory solutions to~(\ref{eq:differentialinclusion}) exist (now with arbitrary $x_0\in\mathbb{R}^n$), provided that the following assumption (Assumption~\ref{ass:K}) holds.

\begin{assumption}[Existence of Carath\'{e}odory solutions]
$\mathcal{F}:\mathbb{R}^n\rightrightarrows\mathbb{R}^n$ has nonempty, compact, and convex values, and is \mbox{\emph{upper semi-continuous}}.
\label{ass:K}
\end{assumption}

\cite{Filippov1988} proved that, for the Filippov set-valued map $\mathcal{F} = \mathcal{K}[F]$, Assumptions~\ref{ass:existenceFilippov} and~\ref{ass:K} are equivalent (with arbitrary $x_0\in\mathbb{R}^n$ in Assumption~\ref{ass:existenceFilippov}). 

Uniqueness of the solution requires further assumptions. Nevertheless, we can characterize the Filippov set-valued map in a similar manner to Clarke's generalized gradient, as seen in the following proposition.

\begin{proposition}[Theorem 1 of~\cite{Paden1987}]
Under Assumption~\ref{ass:existenceFilippov}, we have
\begin{equation}
\mathcal{K}[F](x) = \left\{\lim_{k\to\infty} F(x_k): x_k\in\mathbb{R}^n\setminus (\mathcal{N}_F\cup S) \textnormal{ s.t. } x_k\to x\right\}
\end{equation}
for some (Lebesgue) zero-measure set $\mathcal{N}_F\subset\mathbb{R}^n$ and any other zero-measure set $S\subset\mathbb{R}^n$. In particular, if $F$ is continuous at a fixed $x$, then \mbox{$\mathcal{K}[F](x) = \{F(x)\}$}.
\end{proposition}

For instance, for GF~(\ref{eq:gradientflow}), we have \mbox{$\mathcal{K}[-\nabla f](x) =\{-\nabla f(x)\}$} for every $x\in\mathbb{R}^n$, provided that $f$ is continuously differentiable. Furthermore, if $f$ is only locally  Lipschitz continuous and regular (see Definition~\ref{def:regularfunction}), then $\mathcal{K}[-\nabla f](x) = -\partial f(x)$, where
\begin{equation}
    \partial f(x) \triangleq \left\{\lim_{k\to\infty} \nabla f(x_k): x_k\in \mathbb{R}^n\setminus\mathcal{N}_f \textnormal{ s.t. } x_k\to x\right\}
\end{equation}
denotes Clarke's generalized gradient~\cite{Clarke1981} of $f$, with $\mathcal{N}_f$ denoting the zero-measure set over which $f$ is not differentiable (Rademacher's theorem). It can be established that $\partial f$ coincides with the subgradient of $f$, provided that $f$ is convex. Therefore, the GF~(\ref{eq:gradientflow}) interpreted as Filippov differential inclusion may also be seen as a continuous-time variant of subgradient descent methods.

\section{Finite-Time Stability of Differential Inclusions}
\label{subsec:FTSsemiS}

We are now ready to focus on extending some notions from traditional Lipschitz continuous systems to differential inclusions. 

\begin{definition}\label{equi_point}
We say that $x^\star\in\mathbb{R}^n$ is an \emph{equilibrium} of~(\ref{eq:differentialinclusion}) if $x(t) \equiv x^\star$ on some small enough non-degenerate interval is a Carath\'{e}odory solution to~(\ref{eq:differentialinclusion}). In other words, if and only if $0\in \mathcal{F}(x^\star)$. 
We say that~(\ref{eq:differentialinclusion}) is \emph{(Lyapunov) stable} at $x^\star\in\mathbb{R}^n$ if, for every $\varepsilon > 0$, there exists some $\delta > 0$ such that, for every maximal Carath\'{e}odory solution $x(\cdot)$ of~(\ref{eq:differentialinclusion}), we have $\|x_0 - x^\star\|_2 < \delta \implies \|x(t) - x^\star\|_2 < \varepsilon$ for every $t\geq 0$ in the interval where $x(\cdot)$ is defined. Note that, under Assumption~\ref{ass:K}, if~(\ref{eq:differentialinclusion}) is stable at $x^\star$, then $x^\star$ is an equilibrium of~(\ref{eq:differentialinclusion})~\cite{Bacciotti1999}. Furthermore, we say that~(\ref{eq:differentialinclusion}) is \emph{(locally and strongly) asymptotically stable} at $x^\star\in\mathbb{R}^n$ if is stable at $x^\star$ and there exists some $\delta > 0$ such that, for every maximal Carath\'{e}odory solution $x:[0,\tau)\to\mathbb{R}^n$ of~(\ref{eq:differentialinclusion}), if $\|x_0 - x^\star\|_2 < \delta$ then $x(t)\to x^\star$ as $t\to\tau$. Finally,~(\ref{eq:differentialinclusion}) is \emph{(locally and strongly) finite-time stable} at $x^\star$ if it is asymptotically stable and there exists some $\delta>0$ and $T:B_\delta(x^\star)\to [0,\infty)$ such that, for every maximal Carath\'{e}odory solution $x(\cdot)$ of~(\ref{eq:differentialinclusion}) with $x_0\in B_\delta(x^\star)$, we have $\lim_{t\to T(x_0)}x(t) = x^\star$.
\end{definition}

We will now construct a Lyapunov-based criterion adapted from the literature of finite-time stability of Lipschitz continuous systems.

\begin{lemma}
Let $\mathcal{E}(\cdot)$ be an absolutely continuous function satisfying the differential inequality
\begin{equation}
    \dot{\mathcal{E}}(t) \leq -c\,\mathcal{E}(t)^\alpha
    \label{eq:Edotineqlemma}
\end{equation}
a.e. in $t\geq 0$, with $c,\mathcal{E}(0)>0$ and $\alpha < 1$. Then, there exists some $t^\star > 0$ such that $\mathcal{E}(t)>0$ for $t\in [0,t^\star)$ and $\mathcal{E}(t^\star) = 0$. Furthermore, $t^\star > 0$ can be bounded by
\begin{equation}
    t^\star \leq \frac{\mathcal{E}(0)^{1-\alpha}}{c(1-\alpha)},
    \label{eq:tstarboundlemma}
\end{equation}
with this bound tight whenever~(\ref{eq:Edotineqlemma}) holds with equality. In that case, but now with $\alpha\geq 1$, then $\mathcal{E}(t)>0$ for every $t\geq 0$, with $\lim_{t\to\infty}\mathcal{E}(t) = 0$. This will be represented by $t^\star = \infty$, with $\mathcal{E}(\infty) \triangleq \lim_{t\to\infty}\mathcal{E}(t)$.
\label{lemma:Edotineq}
\end{lemma}

\begin{proof}
Suppose that $\mathcal{E}(t)>0$ for every $t\in [0,T]$ with $T > 0$. Let $t^\star$ be the supremum of all such $T$'s, thus satisfying $\mathcal{E}(t) > 0$ for every $t\in [0,t^\star)$. We now study $\mathcal{E}(t^\star)$. First, by continuity of $\mathcal{E}$, it follows that $\mathcal{E}(t^\star) \geq 0$. Now, by rewriting
\begin{equation}\label{b3}
    \dot{\mathcal{E}}(t) \leq -c\,\mathcal{E}(t)^\alpha \iff \frac{\mathrm{d}}{\mathrm{d}t}\left[\frac{\mathcal{E}(t)^{1-\alpha}}{1-\alpha}\right] \leq -c,
\end{equation}
a.e. in $t\in [0,t^\star)$, we can thus integrate to obtain\footnote{Perhaps the integral of the right-hand-side of the equivalence (\ref{b3}) is not defined everywhere, due to the
possibility that the fractional function is not absolutely continuous everywhere, nevertheless, the transition from
(\ref{b3}) to (\ref{b4}) can be obtained by directly applying the results of Corollary 2.4 of \cite{MNPR20}. We thank
Drs. K. Garg, M. Baranwal, and R. Gupta for bringing this fact to our attention}
\begin{equation}\label{b4}
    \frac{\mathcal{E}(t)^{1-\alpha}}{1-\alpha} - \frac{\mathcal{E}(0)^{1-\alpha}}{1-\alpha} \leq -c\,t,
\end{equation}
everywhere in $t\in [0,t^\star)$, which in turn leads to
\begin{equation}
    \mathcal{E}(t) \leq [\mathcal{E}(0)^{1-\alpha} - c(1-\alpha)t]^{1/(1-\alpha)}
    \label{eq:Vdotinesol}
\end{equation}
and
\begin{equation}
    t \leq \frac{\mathcal{E}(0)^{1-\alpha} - \mathcal{E}(t)^{1-\alpha}}{c(1-\alpha)} \leq \frac{\mathcal{E}(0)^{1-\alpha}}{c(1-\alpha)},
    \label{eq:tboundsol}
\end{equation}
where the last inequality follows from $\mathcal{E}(t)>0$ for every $t\in [0,t^\star)$. Taking the supremum in~(\ref{eq:tboundsol}) then leads to the upper bound~(\ref{eq:tstarboundlemma}). Finally, we conclude that $\mathcal{E}(t^\star) = 0$, since $\mathcal{E}(t^\star) > 0$ is impossible given that it would mean, due to continuity of $\mathcal{E}$, that there exists some $T>t^\star$ such that $\mathcal{E}(t)>0$ for every $t\in [0,T]$, thus contradicting the construction of $t^\star$. 

Finally, notice that if $\mathcal{E}$ is such that~(\ref{eq:Edotineqlemma}) holds with equality, then~(\ref{eq:Vdotinesol}) and the first inequality in~(\ref{eq:tboundsol}) hold with equality as well. The tightness of the bound~(\ref{eq:tstarboundlemma}) thus follows immediately. Furthermore, notice that if~$\alpha\geq 1$, and $\mathcal{E}$ is a tight solution to the differential inequality~(\ref{eq:Edotineqlemma}), \emph{i.e.} $\mathcal{E}(t) = [\mathcal{E}(0)^{1-\alpha} - c(1-\alpha)t]^{1/(1-\alpha)}$, then clearly $\mathcal{E}(t)>0$ for every $t\geq 0$ and $\mathcal{E}(t)\to 0$ as $t\to \infty$.
\qed
\end{proof}

\cite{Cortes2005} proposed (Proposition~2.8) a Lyapunov-based criterion to establish finite-time stability of discontinuous systems, which fundamentally coincides with our Lemma~\ref{lemma:Edotineq} for the particular choice of exponent $\alpha=0$. Their proposition was, however, directly based on Theorem~2 of~\cite{Paden1987}. Later, \cite{Cortes2006} proposed a second-order Lyapunov criterion, which, on the other hand, fundamentally translates to $\mathcal{E}(t) \triangleq V(x(t))$ being strongly convex. Finally, \cite{Hui2009} generalized Proposition~2.8 of~\cite{Cortes2005} in their Corollary~3.1, to establish semistability. Indeed, these two results coincide for isolated equilibria.

We now, by exploiting our Lemma~\ref{lemma:Edotineq}, present a novel result generalizing the aforementioned first-order Lyapunov-based results. More precisely, given a Laypunov candidate function $V(\cdot)$, the objective is to set $\mathcal{E}(t) \triangleq V(x(t))$, and we aim to check that the conditions of Lemma~\ref{lemma:Edotineq} hold. To do this, and assuming $V$ to be locally Lipschitz continuous, we first borrow and adapt from~\cite{Bacciotti1999} the definition of \emph{set-valued time derivative} of $V:\mathcal{D}\to\mathbb{R}$ w.r.t. the differential inclusion~(\ref{eq:differentialinclusion}), given by
\begin{equation}
    \dot{V}(x) \triangleq \{a\in\mathbb{R}: \exists v\in \mathcal{F}(x) \textnormal{ s.t. } a = p\cdot v,\, \forall p\in\partial V(x)\},
\end{equation}
for each $x\in\mathcal{D}$. Notice that, under Assumption~\ref{ass:K} for Filippov differential inclusions $\mathcal{F} = \mathcal{K}[F]$, the set-valued time derivative of $V$ thus coincides with with the set-valued Lie derivative $\mathcal{L}_F V(\cdot)$. Indeed, more generally $\dot{V}$ could be seen as a set-valued Lie derivative $\mathcal{L}_\mathcal{F} V$ w.r.t. the set-valued map $\mathcal{F}$.

\begin{definition}
\label{def:regularfunction}
$V(\cdot)$ is said to be \emph{regular} if every directional derivative, given by
\begin{equation}
    V'(x;v) \triangleq \lim_{h\to 0}\frac{V(x+h\,v) - V(x)}{h},
\end{equation}
exists and is equal to
\begin{equation}
    V^\circ(x;v) \triangleq \limsup_{x'\to x\, h\to 0^+}\frac{V(x'+h\,v) - V(x')}{h},
\end{equation}
known as \emph{Clarke's upper generalized derivative}~\cite{Clarke1981}.
\end{definition}

In practice, regularity is a fairly mild and easy to guarantee condition. For instance, it would suffice that $V$ is convex or continuously differentiable to ensure that it is Lipschitz and regular.

\begin{assumption}
$V:\mathcal{D}\to\mathbb{R}$ is locally Lipscthiz continuous and regular, with $\mathcal{D}\subseteq\mathbb{R}^n$ open.
\label{ass:Vregularity}
\end{assumption}

Under Assumption~\ref{ass:Vregularity}, Clarke's generalized gradient
\begin{equation}
    \partial V(x) \triangleq \{p\in\mathbb{R}^n: V^\circ(x;v) \geq p\cdot v, \forall v\in\mathbb{R}^n\}
\end{equation}
is non-empty for every $x\in\mathcal{D}$, and is also given by
\begin{equation}
    \partial V(x) = \left\{\lim_{k\to\infty}\nabla V(x_k): x_k\in\mathbb{R}^n\setminus\mathcal{N}_V \textnormal{ s.t. } x_k\to x\right\},
\end{equation}
where $\mathcal{N}_V$ denotes the set of points in $\mathcal{D}\subseteq\mathbb{R}^n$ where $V$ is not differentiable (recall Rademacher’s theorem) \cite{Clarke1981}.

Through the following lemma (Lemma~\ref{lemma:dVdtisinVdot}), we can formally establish the correspondence between the set-valued time-derivative of $V$ and the derivative of the energy function $\mathcal{E}(t) \triangleq V(x(t))$ associated with an arbitrary Carath\'{e}odory solution $x(\cdot)$ to the differential inclusion~(\ref{eq:differentialinclusion}).

\begin{lemma}[Lemma~1 of~\cite{Bacciotti1999}]
Under Assumption~\ref{ass:Vregularity}, given any Carath\'{e}odory solution \mbox{$x:[0,\tau)\to\mathbb{R}^n$} to~(\ref{eq:differentialinclusion}), then $\mathcal{E}(t)\triangleq V(x(t))$ is absolutely continuous and $\dot{\mathcal{E}}(t) = \frac{\mathrm{d}}{\mathrm{d}t}V(x(t))\in\dot{V}(x(t))$ a.e. in $t\in [0,\tau)$.
\label{lemma:dVdtisinVdot}
\end{lemma}

We are now ready to state and prove our Lyapunov-based sufficient condition for finite-time stability of differential inclusions.

\begin{theorem}
\label{thm:FTS}
Suppose that Assumptions~\ref{ass:K} and~\ref{ass:Vregularity} hold for some set-valued map $\mathcal{F}:\mathbb{R}^n\rightrightarrows\mathbb{R}^n$ and some function $V:\mathcal{D}\to\mathbb{R}$, where $\mathcal{D}\subseteq\mathbb{R}^n$ is an open and positively invariant neighborhood of a point $x^\star\in\mathbb{R}^n$. Suppose that $V$ is positive definite w.r.t. $x^\star$ and that there exist constants $c>0$ and $\alpha < 1$ such that
\begin{equation}
\sup \dot{V}(x) \leq -c\,V(x)^\alpha
\label{eq:Vdotineq}
\end{equation}
a.e. in $x\in\mathcal{D}$. Then,~(\ref{eq:differentialinclusion}) is finite-time stable at $x^\star$, with settling time satisfying
\begin{equation}
t^\star \leq \frac{V(x_0)^{1-\alpha}}{c(1-\alpha)},
\label{eq:settlingtimebound}
\end{equation}
where $x(0)=x_0$. In particular, any Carath\'{e}odory solution $x(\cdot)$ with $x(0)=x_0\in\mathcal{D}$ will converge in finite time to $x^\star$ under the upper bound~(\ref{eq:settlingtimebound}). Furthermore, if $\mathcal{D}=\mathbb{R}^n$, then~(\ref{eq:differentialinclusion}) is globally finite-time stable. Finally, if $\dot{V}(x)$ is a singleton a.e. in $x\in\mathcal{D}$ and~(\ref{eq:Vdotineq}) holds with equality, then the bound~(\ref{eq:settlingtimebound}) is tight.
\end{theorem}

\begin{proof}
Note that, by Proposition~1 of~\cite{Bacciotti1999}, we know that~(\ref{eq:differentialinclusion}) is Lyapunov stable at $x^\star$. All that remains to be shown is local convergence towards $x^\star$ (which must be an equilibrium) in finite time. Indeed, given any maximal solution $x:[0,t^\star)\to\mathbb{R}^n$ to~(\ref{eq:differentialinclusion}) with $x(0)=x_0\neq x^\star$, we know by Lemma~\ref{lemma:dVdtisinVdot}, that $\mathcal{E}(t) = V(x(t))$ is absolutely continuous with $\dot{\mathcal{E}}(t)\in\dot{V}(x(t))$ a.e. in $t\in [0,t^\star]$. Therefore, we have
\begin{equation}
    \dot{\mathcal{E}}(t) \leq \sup\dot{V}(x(t)) \leq -c\,V(x(t))^\alpha = -c\,\mathcal{E}(t)^\alpha
\end{equation}
a.e. in $t\in [0,t^\star]$. Since $\mathcal{E}(0)=V(x_0) > 0$, given that $x_0\neq x^\star$, the result then follows by invoking Lemma~\ref{lemma:Edotineq} and noting that $\mathcal{E}(t^\star) = 0 \iff V(t^\star,x(t^\star)) = 0 \iff x(t^\star) = x^\star$.
\qed
\end{proof}

Finite-time stability still follows without Assumption~\ref{ass:K}, provided that $x^\star$ is an equilibrium of~(\ref{eq:differentialinclusion}). In practical terms, this means that trajectories starting arbitrarily close to $x^\star$ may not actually exist, but nevertheless there exists a neighborhood $\mathcal{D}$ of $x^\star$ over which, any trajectory $x(\cdot)$ that indeed exists and starts at $x(0)=x_0\in\mathcal{D}$ must converge in finite time to $x^\star$, with settling time upper bounded by $T(x_0)$ (the bound still tight in the case that~(\ref{eq:Vdotineq}) holds with equality). 

\section{Proof of Theorem~\ref{thm1}}\label{section-theorem1-proof}
\begin{proof}
    Let us focus on the $q$-RGF~(\ref{eq:qRGF}) (the case of $q$-SGF~(\ref{eq:qSGF}) follows exactly the same steps) with the candidate Lyapunov function $V \triangleq f - f(x^\star)$. Clearly, $V$ is Lipschitz continuous and regular (given that it is continuously differentiable). Furthermore, $V$ is positive definite w.r.t. $x^\star$. 

    Notice that, due to the dominated gradient assumption, $x^\star$ must be an isolated stationary point of $f$. To see this, notice that, if $x^\star$ were not an isolated stationary point, then there would have to exist some $\tilde{x}^\star$ sufficiently near $x^\star$ such that $\tilde{x}^\star$ is both a stationary point of~$f$, and satisfies $f(\tilde{x}^\star) > f(x^\star)$, since $x^\star$ is a strict local minimizer of~$f$. But then, we would have
    \begin{equation}
        0 = \frac{p-1}{\psi}\|\nabla f(\tilde{x}^\star)\|_2^{\frac{\psi}{p-1}} \geq \mu^{\frac{1}{p-1}}(f(\tilde{x}^\star) - f(x^\star)) > 0,
    \end{equation}
    and subsequently $0>0$, which is absurd. 
    
    Therefore, $F(x) \triangleq -c\nabla f(x)/\|\nabla f(x)\|_2^{\frac{q-2}{q-1}}$ is continuous for every $x\in\mathcal{D}\setminus\{0\}$, for some small enough open neighborhood~$\mathcal{D}$ of $x^\star$. Let us assume that $\mathcal{D}$ is positively invariant w.r.t.~(\ref{eq:qRGF}), which can be achieved, for instance, by replacing $\mathcal{D}$ with its intersection with some small enough strict sublevel set of $f$. Notice that 
    $\|F(x)\|_2 = c\|\nabla f(x)\|_2^{\frac{1}{q-1}}$
    with $q\in (p,\infty]\subset (1,\infty]$, \emph{i.e.}, $\frac{1}{q-1} \in [0,\infty)$. If $q=\infty$, which results in the normalized gradient flow 
    $\dot{x} = -\frac{\nabla f(x)}{\|\nabla f(x)\|_2}$ 
    proposed by~\cite{Cortes2006}, then $\|F(x)\|_2 = c > 0$. We can thus show that $F(x)$ is discontinuous at $x=0$ for $q=\infty$. On the other hand, if $q \in (p,\infty) \subset (1,\infty)$, then we have $\|F(x)\|_2\to 0$ as $x\to x^\star$, and thus $F(x)$ is continuous (but not Lipschitz) at $x=x^\star$. Regardless, we may freely focus exclusively on $\mathcal{D}\setminus\{x^\star\}$ since $\{x^\star\}$ is obviously a zero-measure set.

    Let $\mathcal{F}\triangleq\mathcal{K}[F]$. We thus have, for each $x\in\mathcal{D}\setminus\{x^\star\}$,
    \begin{subequations}
    \begin{align}
        \sup \dot{V}(x) &= \sup\,\{a\in\mathbb{R}: \exists v\in\mathcal{F}(x) \textnormal{ s.t. } a=p\cdot v, \forall p\in\partial V(x)\}\nonumber \\
        &= \sup\,\{\nabla V(x)\cdot v: v\in\mathcal{F}(x)\}\nonumber\\
        &= \nabla V(x)\cdot F(x)
        = -c\|\nabla f(x)\|_2^{2-\frac{q-2}{q-1}}= -c\|\nabla f(x)\|_2^{\frac{1}{\theta'}}\nonumber\\
        &\leq -c [C(f(x)-f(x^\star))^\theta]^{\frac{1}{\theta'}}
        = -c C^{\frac{1}{\theta'}}V(x)^{\frac{\theta}{\theta'}}.\label{ineedyou}
    \end{align}
    \end{subequations}
    Since $\frac{\theta}{\theta'}<1$, given that $s>1\mapsto\frac{s-1}{s}$ is strictly increasing, then the conditions of Theorem~\ref{thm:FTS} are satisfied. In particular, we have finite-time stability at $x^\star$ with a settling time $t^\star$ upper bounded by
    \begin{equation}
        t^\star \leq \frac{(f(x_0) - f(x^\star))^{1-\frac{\theta}{\theta'}}}{c\,C^{\frac{1}{\theta'}}\left(1-\frac{\theta}{\theta'}\right)} \leq \frac{(\|\nabla f(x_0)\|_2/C)^{\frac{1}{\theta}\left(1-\frac{\theta}{\theta'}\right)}}{c\,C^{\frac{1}{\theta'}}\left(1-\frac{\theta}{\theta'}\right)} = \frac{\|\nabla f(x_0)\|_2^{\frac{1}{\theta}-\frac{1}{\theta'}}}{cC^\frac{1}{\theta}\left(1-\frac{\theta}{\theta'}\right)}
    \end{equation}
    for each $x_0\in\mathcal{D}$, which completes the proof.
    \qed
\end{proof}

\newpage
\bibliographystyle{plainnat}
\bibliography{references}

\end{document}